\documentclass{article}

\usepackage[final,nonatbib]{nips_2016}
\usepackage[style=numeric,
            sorting=none,
            uniquelist=false,
            maxnames=1,
            maxbibnames=10,
            backend=biber,
            natbib=true,
            url=true,
            doi=true,
            eprint=true]{biblatex}
\usepackage[disable]{todonotes}

\usepackage[utf8]{inputenc} % allow utf-8 input
\usepackage[T1]{fontenc}    % use 8-bit T1 fonts
\usepackage{hyperref}       % hyperlinks
\usepackage{url}            % simple URL typesetting
\usepackage{booktabs}       % professional-quality tables
\usepackage{amsfonts}       % blackboard math symbols
\usepackage{nicefrac}       % compact symbols for 1/2, etc.
\usepackage{microtype}      % microtypography

\usepackage{graphicx}
\usepackage[subrefformat=parens,labelformat=parens]{subcaption}

\usepackage{algorithm}
\usepackage[noend]{algpseudocode}
\usepackage{hyperref}
\definecolor{mydarkblue}{rgb}{0,0.08,0.45}
\hypersetup{
    colorlinks=true,
    linkcolor=mydarkblue,
    citecolor=mydarkblue,
    filecolor=mydarkblue,
    urlcolor=mydarkblue}

\algrenewcommand{\algorithmicrequire}{\textbf{Input:}}
\algrenewcommand{\algorithmicensure}{\textbf{Output:}}
\algrenewcommand\Return{\State\algorithmicreturn{}~}
\algrenewcommand\algorithmicindent{1.0em}%

\usepackage{amsmath,amsthm,amssymb,amsfonts}
\usepackage{mathtools}
\mathtoolsset{showonlyrefs}

\usepackage{placeins}

\newcommand{\R}{\mathbb{R}}
\newcommand{\T}{\mathsf{T}}
\newcommand{\N}{\mathcal{N}}

\newcommand{\E}{\mathbb{E}}

\newcommand{\X}{\mathcal{X}}
\newcommand{\B}{\mathcal{B}}
\renewcommand{\L}{\mathcal{L}}
\newcommand{\Lsvae}{\mathcal{L}_{\mathrm{SVAE}}}
\newcommand{\iid}[1]{\stackrel{\mathrm{iid}}{#1}}

\DeclareMathOperator{\KL}{KL}
\DeclareMathOperator{\range}{range}
\DeclareMathOperator{\MGF}{M}
\DeclareMathOperator*{\argmin}{arg\,min}
\DeclareMathOperator*{\argmax}{arg\,max}
\newcommand{\given}{\, | \,}
\newcommand{\opt}[1]{{#1}^{*} \!}

\newtheoremstyle{break}
  {\topsep}{\topsep}%
  {\itshape}{}%
  {\bfseries}{}%
  {\newline}{}%
\theoremstyle{break}
\newtheorem{theorem}{Theorem}[section]
\newtheorem{proposition}[theorem]{Proposition}

\newtheorem{lemma}[theorem]{Lemma}
\newtheorem{corollary}[theorem]{Corollary}
\newtheorem{definition}[theorem]{Definition}

\newcommand{\obsparam}{\gamma}
\newcommand{\vareta}{\eta}
\newcommand{\prioreta}{\eta^0}
\newcommand{\sectwoeta}{\eta}
\newcommand{\natnabla}{{\widetilde \nabla}}

\title{Composing graphical models with neural networks\\ for structured representations and fast inference}

\author{
  Matthew James Johnson\\
  Harvard University\\
  \texttt{mattjj@seas.harvard.edu}
  \And
  David Duvenaud\\
  Harvard University\\
  \texttt{dduvenaud@seas.harvard.edu}
  \And
  Alexander B. Wiltschko\\
  Harvard University, Twitter\\
  \texttt{awiltsch@fas.harvard.edu}
  \And
  Sandeep R. Datta\\
  Harvard Medical School\\
  \texttt{srdatta@hms.harvard.edu}
  \And
  Ryan P. Adams\\
  Harvard University, Twitter\\
  \texttt{rpa@seas.harvard.edu}
}

\begin{document}
\maketitle

\begin{abstract}
We propose a general modeling and inference framework that combines the complementary strengths of probabilistic graphical models and deep learning methods. Our model family composes latent graphical models with neural network observation likelihoods. For inference, we use recognition networks to produce local evidence potentials, then combine them with the model distribution using efficient message-passing algorithms. All components are trained simultaneously with a single stochastic variational inference objective. We illustrate this framework by automatically segmenting and categorizing mouse behavior from raw depth video, and demonstrate several other example models.
\end{abstract}

\section{Introduction}
\label{sec:intro}

Modeling often has two goals: first, to learn a flexible representation of 
complex high-dimensional data, such as images or speech recordings, and second, to find structure that is interpretable and generalizes to new tasks.
Probabilistic graphical models
\citep{koller2009probabilistic,murphy2012machine} provide many tools to build
structured representations, but often make rigid assumptions and may require
significant feature engineering.
Alternatively, deep learning methods allow flexible data representations to be
learned automatically, but may not directly encode interpretable or tractable
probabilistic structure.
Here we develop a general modeling and inference framework that combines these
complementary strengths.

Consider learning a generative model for video of a mouse.
Learning interpretable representations for such data, and comparing them as the
animal's genes are edited or its brain chemistry altered, gives useful
behavioral phenotyping tools for neuroscience and for high-throughput drug
discovery~\citep{wiltschko2015mapping}.
Even though each image is encoded by hundreds of pixels, the data lie near a low-dimensional nonlinear manifold.
A useful generative model must not only learn this manifold but also provide an
interpretable representation of the mouse's behavioral dynamics.
A natural representation from ethology \citep{wiltschko2015mapping} is that the
mouse's behavior is divided into brief, reused actions, such as darts, rears,
and grooming bouts.
Therefore an appropriate model might switch between discrete states, with each
state representing the dynamics of a particular action.
These two learning tasks~---~identifying an image manifold and a structured dynamics model~---~are complementary: we want to learn the image manifold in terms
of coordinates in which the structured dynamics fit well.
A similar challenge arises in speech \citep{hinton2012deep}, where
high-dimensional spectrographic data lie near a low-dimensional manifold because they are
generated by a physical system with relatively few degrees of freedom
\citep{deng1999computational} but also include the discrete latent dynamical
structure of phonemes, words, and grammar \citep{deng2004switching}.

To address these challenges, we propose a new framework to design and learn
models that couple nonlinear likelihoods with structured latent variable representations.
Our approach uses graphical models for representing structured probability
distributions while enabling fast exact inference subroutines, and uses ideas from
variational autoencoders \citep{kingma2013autoencoding,rezende2014stochastic} for learning not only
the nonlinear feature manifold but also bottom-up recognition networks to
improve inference.
Thus our method enables the combination of flexible deep learning feature
models with structured Bayesian (and even nonparametric \citep{johnson2014svi})
priors.
Our approach yields a single variational inference objective in which all
components of the model are learned simultaneously.
Furthermore, we develop a scalable fitting algorithm that combines several
advances in efficient inference, including stochastic variational inference
\citep{hoffman2013stochastic}, graphical model message
passing~\citep{koller2009probabilistic}, and backpropagation with the
reparameterization trick \citep{kingma2013autoencoding}.
Thus our algorithm can leverage conjugate exponential family structure where
it exists to efficiently compute natural gradients with respect to some
variational parameters, enabling effective second-order optimization
\citep{martens2015perspectives}, while using backpropagation to
compute gradients with respect to all other parameters.
We refer to our general approach as the structured variational autoencoder (SVAE).

\section{Latent graphical models with neural net observations}
\label{sec:models}

In this paper we propose a broad family of models.
Here we develop three specific examples.

\subsection{Warped mixtures for arbitrary cluster shapes}
\label{sec:warped_mixture}

One particularly natural structure used frequently in graphical models is the
discrete mixture model.
By fitting a discrete mixture model to data, we can discover natural clusters
or units. These discrete structures are difficult to represent directly in
neural network models.

Consider the problem of modeling the data $y = \{y_n\}_{n=1}^N$ shown in
Fig.~\ref{fig:spiral data}.
A standard approach to finding the clusters in data is to fit a Gaussian
mixture model (GMM) with a conjugate prior:
\begin{equation}
  \pi \sim \textnormal{Dir}(\alpha),
  \quad
  \;
  (\mu_k, \Sigma_k) \iid\sim \textnormal{NIW}(\lambda),
  \quad
  \;
  z_n \given \pi \iid\sim \pi
  \quad
  \;
  y_n \given z_n, \{(\mu_k, \Sigma_k)\}_{k=1}^K \iid\sim \N(\mu_{z_n}, \Sigma_{z_n}).
  \notag
\end{equation}
However, the fit GMM does not represent the natural clustering of the data
(Fig.~\ref{fig:gmm density}).
Its inflexible Gaussian observation model limits its ability to
parsimoniously fit the data and their natural semantics.
% The same observation has been made in speech modeling \citep{hinton2012deep,
% deng1999computational, deng2004switching}.

Instead of using a GMM, a more flexible alternative would be a neural network
density model:
\begin{equation}
  \obsparam \sim p(\obsparam)
  \qquad
  x_n \iid\sim \N(0,I),
  \qquad
  y_n \given x_n, \obsparam \iid\sim \N(\mu(x_n; \obsparam), \, \Sigma(x_n; \obsparam)),
  \label{eq:density_network}
\end{equation}
where $\mu(x_n; \obsparam)$ and $\Sigma(x_n; \obsparam)$ depend on $x_n$
through some smooth parametric function, such as multilayer perceptron
(MLP), and where $p(\obsparam)$ is a Gaussian prior~\cite{mackay1999density}.
This model fits the data density well (Fig. \ref{fig:density network})
but does not explicitly represent discrete mixture components,
which might provide insights into the data or natural units for generalization.
See Fig.~\ref{fig:vae pgm} for a graphical model.

\begin{figure}[tb]
  \centering
  \begin{subfigure}[t]{0.245\textwidth}
    \centering
    \includegraphics[width=\textwidth]{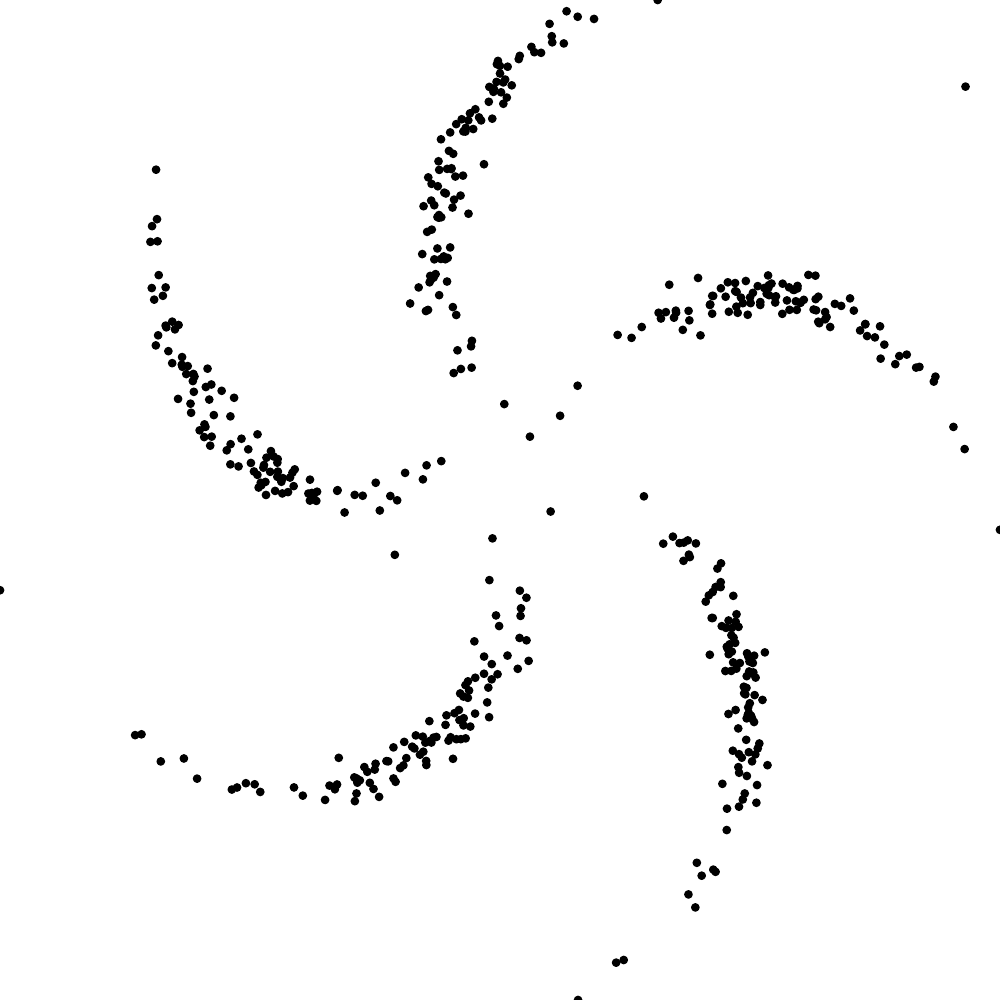}
    \caption{Data}
    \label{fig:spiral data}
  \end{subfigure}
   \begin{subfigure}[t]{0.245\textwidth}
    \centering
    \includegraphics[width=\textwidth]{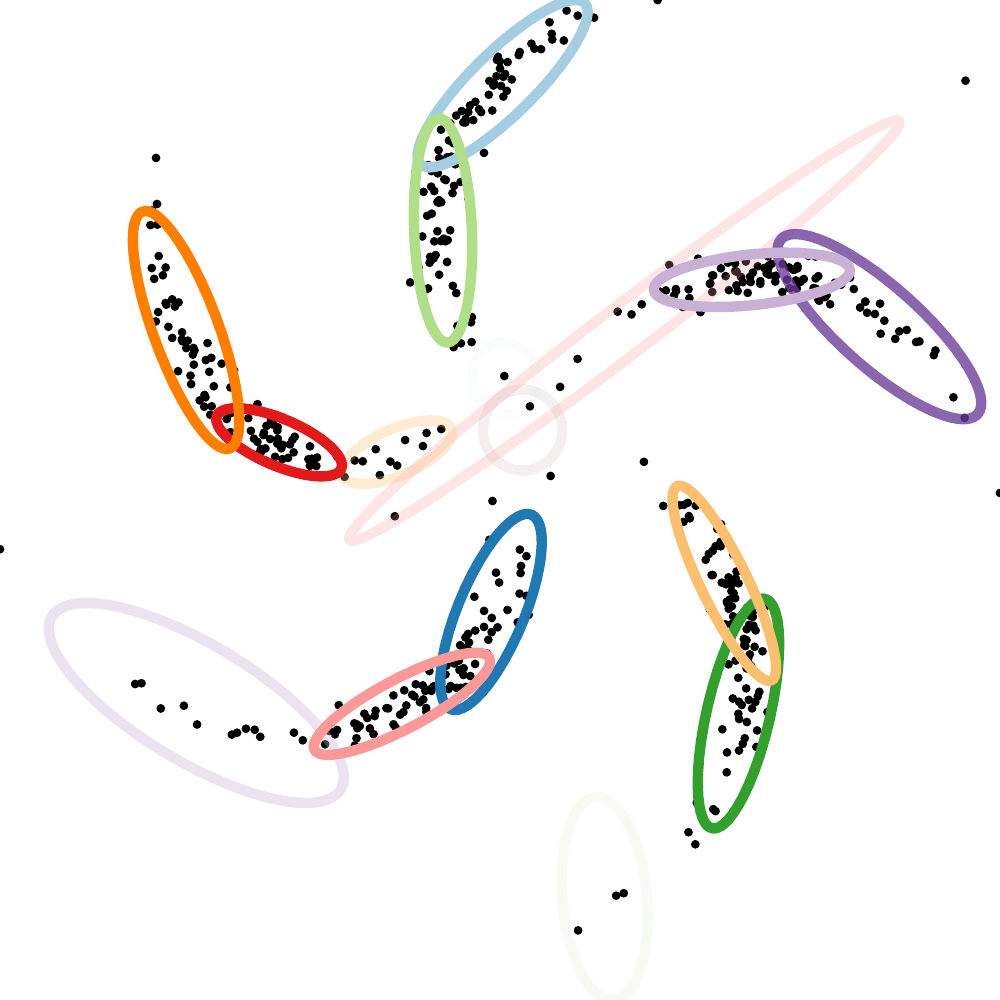}
    \caption{GMM}
    \label{fig:gmm density}
  \end{subfigure}
  \begin{subfigure}[t]{0.245\textwidth}
    \centering
    \includegraphics[width=\textwidth]{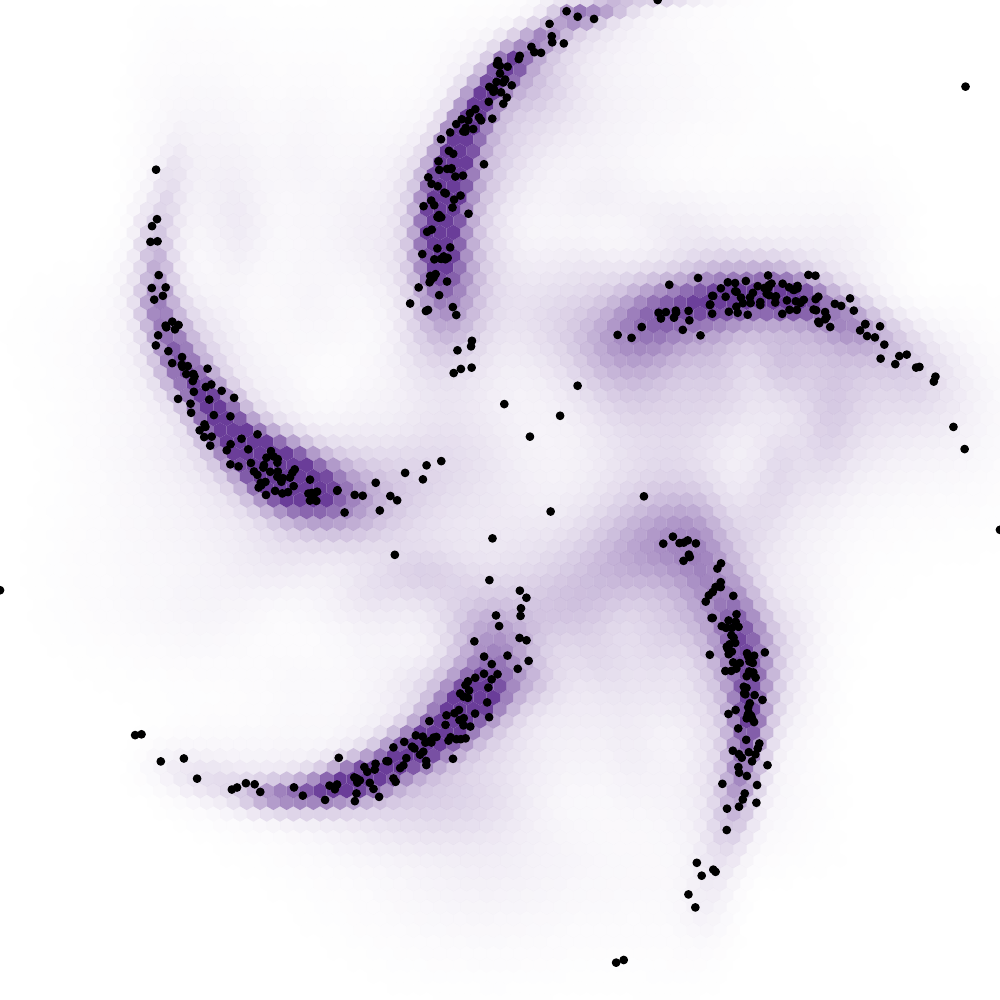}
    \caption{Density net (VAE)}
    \label{fig:density network}
  \end{subfigure}
  \begin{subfigure}[t]{0.245\textwidth}
    \centering
    \includegraphics[width=\textwidth]{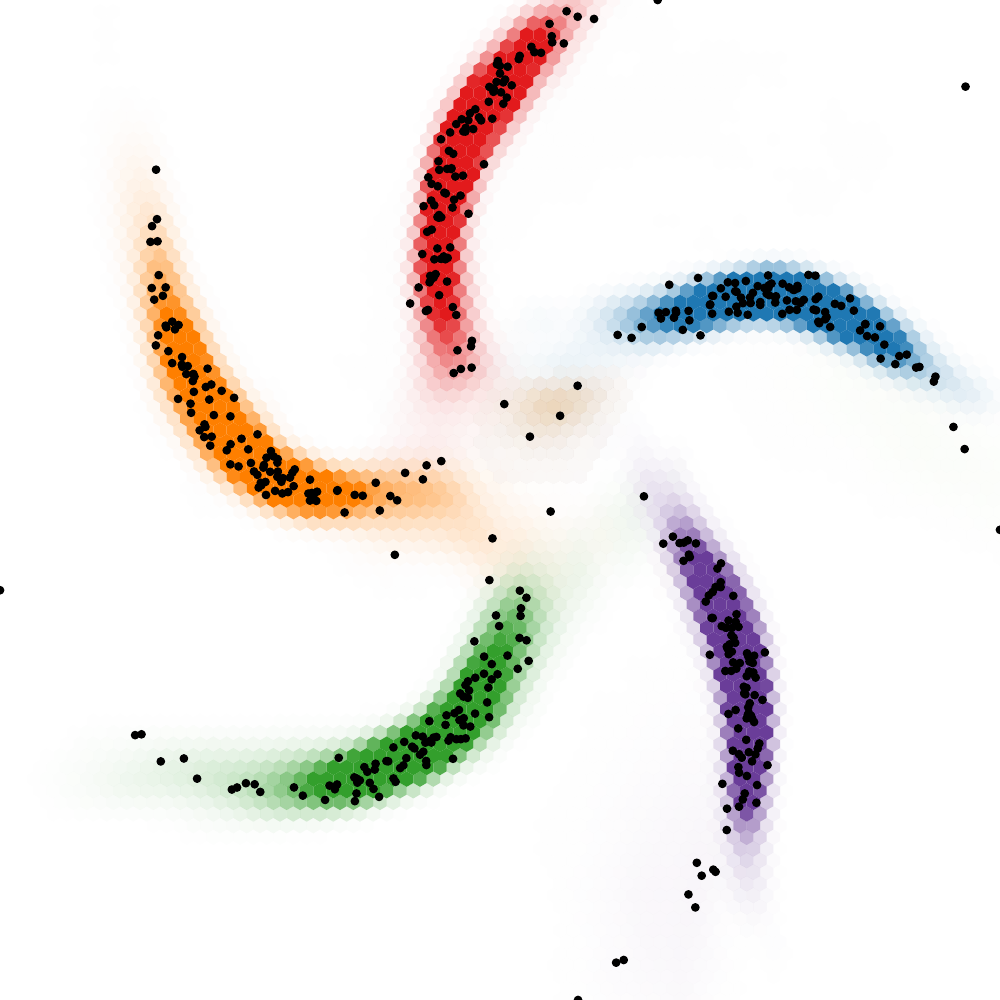}
    \caption{GMM SVAE}
    \label{fig:warped density}
  \end{subfigure}
  \caption{
    Comparison of generative models fit to spiral cluster data. See Section~\ref{sec:warped_mixture}.
    % A comparison of different models.
    % A standard mixture of Gaussians requires several components to fit each cluster.
    % A density network warps a single Gaussian to fit the data density, but does
    % not explicitly separate clusters. A warped mixture of Gaussians places a
    % separate, flexibly-shaped component on each data cluster.
  }
  \label{fig:densities}
\end{figure}

By composing a latent GMM with nonlinear observations, we can combine the
modeling strengths of both~\citep{IwaDuvGha2012warped}, learning both discrete clusters along with
non-Gaussian cluster shapes:
\begin{gather}
  \pi \sim \textnormal{Dir}(\alpha),
  \qquad
  (\mu_k, \Sigma_k) \iid\sim \textnormal{NIW}(\lambda),
  \qquad
  \obsparam \sim p(\obsparam)
  \\
  z_n \given \pi \iid\sim \pi
  \qquad
  x_n \iid\sim \N(\mu^{(z_{n})}, \Sigma^{(z_{n})}),
  \qquad
  y_n \given x_n, \obsparam \iid\sim \N(\mu(x_n; \obsparam), \, \Sigma(x_n; \obsparam)).
\end{gather}
This combination of flexibility and structure is shown in Fig.~\ref{fig:warped density}.
See Fig.~\ref{fig:gmm pgm} for a graphical model.
% \todo{observation: if the neural network didn't have any nonlinearities, you'd get a mixture of factor analyzers, e.g., old paper by Ghahramani and Hinton. Also deep mixtures of factor analyzers with Charlie Tang.}

\subsection{Latent linear dynamical systems for modeling video}
\label{sec:lds}

\begin{figure}[tb]
  \centering
  \begin{subfigure}[b]{0.18\textwidth}
    \centering
    \includegraphics[scale=0.6]{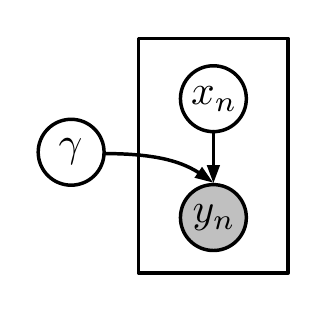}
    \caption{Latent Gaussian}
    \label{fig:vae pgm}
  \end{subfigure}
  \begin{subfigure}[b]{0.15\textwidth}
    \centering
    \includegraphics[scale=0.6]{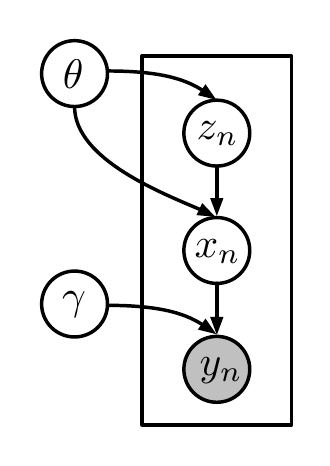}
    \caption{Latent GMM}
    \label{fig:gmm pgm}
  \end{subfigure}
  \hfill
  \begin{subfigure}[b]{0.32\textwidth}
    \centering
    \includegraphics[scale=0.6]{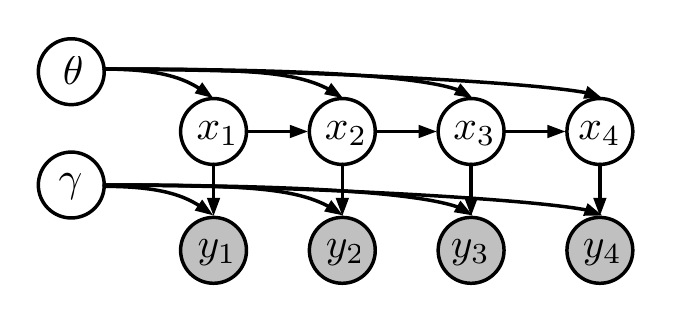}
    \caption{Latent LDS}
    \label{fig:lds pgm}
  \end{subfigure}
  \hfill
  \begin{subfigure}[b]{0.32\textwidth}
    \centering
    \includegraphics[scale=0.6]{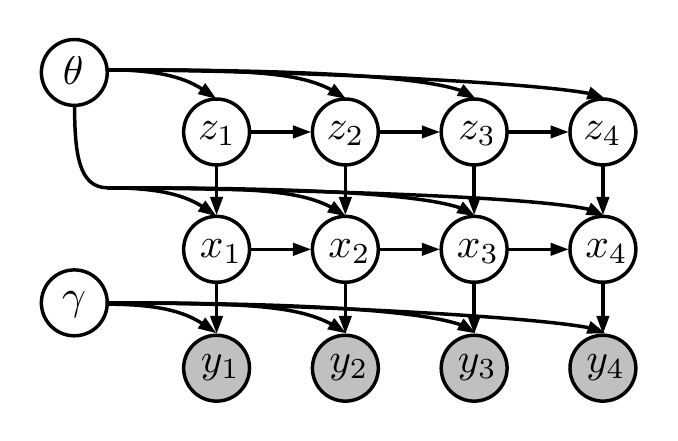}
    \caption{Latent SLDS}
    \label{fig:slds pgm}
  \end{subfigure}
  \caption{Generative graphical models discussed in Section~\ref{sec:models}.}
  \label{fig:graphical_models}
\end{figure}

Now we consider a harder problem: generatively modeling video.
Since a video is a sequence of image frames, a natural place to start is with a
model for images.
% Recent advances in deep learning have produced excellent generative image
% models.
\citet{kingma2013autoencoding} shows that the density network of
Eq.~\eqref{eq:density_network} can accurately represent a dataset of
high-dimensional images $\{y_n\}_{n=1}^N$ in terms of the low-dimensional latent
variables $\{x_n\}_{n=1}^N$, each with independent Gaussian distributions.

To extend this image model into a model for videos, we can introduce
dependence through time between the latent Gaussian samples $\{x_n\}_{n=1}^N$.
For instance, we can make each latent variable $x_{n}$ depend on the previous
latent variable $x_{n-1}$ through a Gaussian linear dynamical system, writing
\begin{equation}
  x_{n} = A x_{n-1} + B u_n, \qquad u_n \iid\sim \N(0, I), \qquad A,B \in \R^{m \times m},
\end{equation}
where the matrices $A$ and $B$ have a conjugate prior.
This model has low-dimensional latent states and dynamics as well as a rich
nonlinear generative model of images.
In addition, the timescales of the dynamics are
represented directly in the eigenvalue spectrum of $A$, providing both
interpretability and a natural way to encode prior information.
See Fig.~\ref{fig:lds pgm} for a graphical model.

\subsection{Latent switching linear dynamical systems for parsing behavior from video}
\label{sec:slds}

As a final example that combines both time series structure and discrete latent
units, consider again the behavioral phenotyping problem described in
Section~\ref{sec:intro}.
Drawing on graphical modeling tools, we can construct a latent switching linear
dynamical system (SLDS) \citep{fox2011slds} to represent the data in terms of continuous latent
states that evolve according to a discrete library of linear dynamics,
and drawing on deep learning methods we can generate video frames with a neural
network image model.

At each time $n \in \{1,2,\ldots,N\}$ there is a discrete-valued latent state
$z_n \in \{1,2,\ldots,K\}$ that evolves according to Markovian dynamics.
The discrete state indexes a set of linear dynamical parameters, and
the continuous-valued latent state $x_n \in \R^m$ evolves according to the
corresponding dynamics,
\begin{equation}
  z_{n} \given z_{n-1}, \pi \sim \pi_{z_{n-1}} , \qquad
  x_{n} = A_{z_{n}} x_{n-1} + B_{z_{n}} u_n, \qquad u_n \iid\sim\N(0,I),
\end{equation}
where $\pi = \{\pi_k\}_{k=1}^K$ denotes the Markov transition matrix and
$\pi_k \in \R^K_+$ is its $k$th row.
We use the same neural net observation model as
in Section~\ref{sec:lds}.
This SLDS model combines both continuous and discrete latent variables with
rich nonlinear observations.
See Fig.~\ref{fig:slds pgm} for a graphical model.

\section{Structured mean field inference and recognition networks}
\label{sec:inference}

Why aren't such rich hybrid models used more frequently?
The main difficulty with combining rich latent variable structure and flexible
likelihoods is inference.
The most efficient inference algorithms used in graphical models, like
structured mean field and message passing, depend on conjugate exponential
family likelihoods to preserve tractable structure.
When the observations are more general, like neural network models, inference
must either fall back to general algorithms that do not exploit the
model structure or else rely on bespoke algorithms developed for one model
at a time.

In this section, we review inference ideas from conjugate exponential family probabilistic graphical models and variational autoencoders, which we combine and generalize in the next section.
%show how to exploit graphical model structure for fast
%variational inference in all of these models.

\subsection{Inference in graphical models with conjugacy structure}

Graphical models and exponential families provide many algorithmic tools for
efficient inference~\citep{wainwright2008graphical}.
Given an exponential family latent variable model, when the observation model
is a conjugate exponential family, the conditional distributions stay in the
same exponential families as in the prior and hence allow for the same
efficient inference algorithms.

For example, consider learning a Gaussian linear dynamical system model with
linear Gaussian observations.
The generative model for latent states $x = \{x_n\}_{n=1}^N$ and observations
$y = \{y_n\}_{n=1}^N$ is
\begin{equation}
  x_n = A x_{n-1} + B u_{n},
  \qquad
  u_n \iid\sim \N(0, I),
  \qquad
  y_n = C x_n + D v_n,
  \qquad
  v_n \iid\sim \N(0,I),
\end{equation}
given parameters ${\theta = (A, B, C, D)}$ with a conjugate prior $p(\theta)$.
To approximate the posterior~$p(\theta, x \given y)$, consider the mean field
family $q(\theta)q(x)$ and the variational inference objective
\begin{equation}
  \L[ \, q(\theta) q(x) \, ] = \E_{q(\theta)q(x)} \! \left[ \log \frac{p(\theta) p(x \given \theta) p(y \given x, \theta)}{q(\theta)q(x)} \right],
  \label{eq:lds_vlb}
\end{equation}
where we can optimize the variational family $q(\theta)q(x)$ to approximate the
posterior $p(\theta, x \given y)$ by maximizing Eq.~\eqref{eq:lds_vlb}.
Because the observation model $p(y \given x, \theta)$ is conjugate to the
latent variable model $p(x \given \theta)$, for any fixed $q(\theta)$ the optimal factor
$\opt q(x) \triangleq \argmax_{q(x)} \L[ \, q(\theta) q(x) \, ]$ is itself a
Gaussian linear dynamical system with parameters that are simple functions of
the expected statistics of $q(\theta)$ and the data $y$.
As a result, for fixed $q(\theta)$ we can easily compute $\opt q(x)$ and use
message passing algorithms to perform exact inference in it.
However, when the observation model is not conjugate to the latent variable
model, these algorithmically exploitable structures break down.

\subsection{Recognition networks in variational autoencoders}
\label{sec:vae_inference}

The variational autoencoder (VAE) \citep{kingma2013autoencoding} handles
general non-conjugate observation models by introducing recognition networks.
%The VAE uses the density network model of Section~\ref{sec:models}, which pairs a simple i.i.d.~Gaussian latent variable model $p(x)$ with a general nonlinear observation model $p(y \given x, \obsparam)$.
For example, when a Gaussian latent variable model $p(x)$ is paired with a
general nonlinear observation model $p(y \given x, \obsparam)$, the posterior
$p(x \given y, \obsparam)$ is non-Gaussian, and it is difficult to compute an
optimal Gaussian approximation.
The VAE instead learns to directly output a suboptimal Gaussian factor $q(x
\given y)$ by fitting a parametric map from data $y$ to a mean and covariance,
$\mu(y; \phi)$ and $\Sigma(y; \phi)$, such as an MLP with parameters $\phi$.
By optimizing over $\phi$, the VAE effectively learns how to condition on
non-conjugate observations $y$ and produce a good approximating factor.

%Thus the variational autoencoder and its extensions
%\citep{archer2015black,krishnan2015deep} fit recognition models that directly
%output an approximate posterior in a bottom-up way.
%The SVAE generalizes this idea by using recognition networks to instead output
%conjugate graphical model potentials.

\begin{figure}[tb]
  \centering
  \begin{subfigure}[b]{0.15\textwidth}
    \centering
    \includegraphics[scale=0.6]{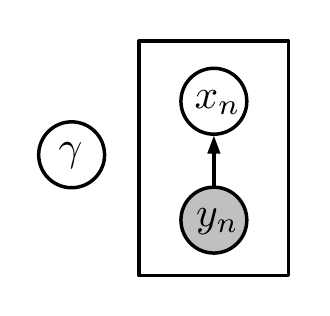}
    \caption{VAE}
  \end{subfigure}
  \begin{subfigure}[b]{0.15\textwidth}
    \centering
    \includegraphics[scale=0.6]{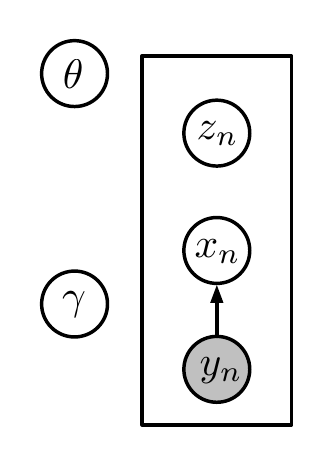}
    \caption{GMM SVAE}
  \end{subfigure}
  \hfill
  \begin{subfigure}[b]{0.32\textwidth}
    \centering
    \includegraphics[scale=0.6]{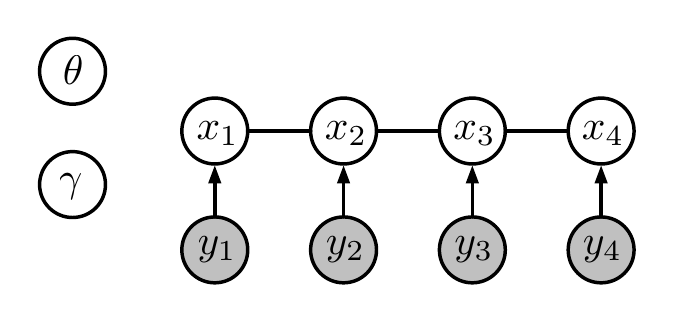}
    \caption{LDS SVAE}
  \end{subfigure}
  \hfill
  \begin{subfigure}[b]{0.32\textwidth}
    \centering
    \includegraphics[scale=0.6]{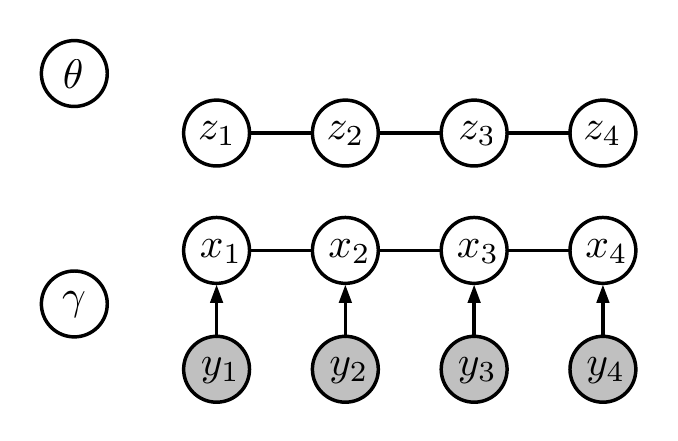}
    \caption{SLDS SVAE}
  \end{subfigure}
  \caption{Variational families and recognition networks for the VAE \citep{kingma2013autoencoding} and three SVAE examples.}
  \label{fig:variational_families}
\end{figure}

\section{Structured variational autoencoders}
%The main idea is to learn recognition networks that output conjugate graphical
%model potentials, and then take advantage of efficient message-passing algorithms for the
%latent graphical model to combine all the evidence from the potentials with the
%prior structure in the model.

We can combine the tractability of conjugate graphical model inference with the flexibility of variational autoencoders.
%The SVAE combines the inference methods used for conjugate-observation graphical models and variational autoencoders.
The main idea is to use a conditional random field (CRF) variational family.
We learn recognition networks that output conjugate graphical model potentials
instead of outputting the complete variational distribution's parameters directly.
These potentials are then used in graphical model inference algorithms in place
of the non-conjugate observation likelihoods.

The SVAE algorithm computes stochastic gradients of a mean field variational
inference objective.
It can be viewed as a generalization both of the natural gradient SVI algorithm
for conditionally conjugate models \citep{hoffman2013stochastic} and of the
AEVB algorithm for variational autoencoders \citep{kingma2013autoencoding}.
Intuitively, it proceeds by sampling a data minibatch, applying the recognition
model to compute graphical model potentials, and using graphical model
inference algorithms to compute the variational factor, combining the
evidence from the potentials with the prior structure in the model.
This variational factor is then used to compute gradients of the mean field
objective.
See Fig.~\ref{fig:variational_families} for graphical models of the variational
families with recognition networks for the models developed in
Section~\ref{sec:models}.

In this section, we outline the SVAE model class more formally, write the mean
field variational inference objective, and show how to efficiently compute
unbiased stochastic estimates of its gradients.
The resulting algorithm for computing gradients of the mean field objective,
shown in Algorithm~\ref{alg:svae}, is simple and efficient and can be readily
applied to a variety of learning problems and graphical model structures.
See the supplementals for details and proofs.

\begin{algorithm}[tb]
  \caption{Estimate SVAE lower bound and its gradients}
  \begin{algorithmic}
    \Require Variational parameters $(\vareta_\theta, \vareta_\obsparam, \phi)$, data sample $y$
    \Function{SVAEGradients}{$\vareta_\theta$, $\vareta_\obsparam$, $\phi$, $y$}
    \State $\psi \gets r(y_n; \phi)$
    \Comment Get evidence potentials

    \State $(\hat x, \; \bar t_x, \; \KL^{\textnormal{local}}) \gets \Call{PGMInference}{\vareta_\theta, \psi}$
    \Comment Combine evidence with prior

    \State $\hat \obsparam \sim q(\obsparam)$ $\phantom{A^l}$
    \Comment Sample observation parameters

    \State $\L \gets N \log p(y \given \hat x, \hat \obsparam) - N \KL^{\textnormal{local}} - \KL(q(\theta)q(\obsparam) \| p(\theta)p(\obsparam))$
    \Comment Estimate variational bound

    \State $\natnabla_{\vareta_\theta} \L \gets \prioreta_\theta - \vareta_\theta + N(\bar t_x, 1) + N (\nabla_{\vareta_x} \log p(y \given \hat x, \hat \obsparam), 0)$
    \Comment Compute natural gradient

    \Return lower bound $\L$, natural gradient $\natnabla_{\vareta_\theta} \L$, ~gradients $\nabla_{\vareta_\obsparam, \phi} \L$
    \EndFunction

    \Function{PGMInference}{$\vareta_\theta$, $\psi$}
    \State $\opt q(x) \gets \Call{OptimizeLocalFactors}{\vareta_\theta, \psi}$
    \Comment Fast message-passing inference
    \Return sample $\hat{x} \sim \opt q(x)$, ~statistics $\E_{\opt q(x)} t_x(x)$, ~divergence $\E_{q(\theta)} \KL(\opt q(x) \| p(x \given \theta))$
    \EndFunction

  \end{algorithmic}
  \label{alg:svae}
\end{algorithm}

\subsection{SVAE model class}
\label{sec:svae_objective}

To set up notation for a general SVAE, we first define a conjugate pair of
exponential family densities on global latent variables $\theta$ and local
latent variables $x = \{x_n\}_{n=1}^N$.
Let $p(x \given \theta)$ be an exponential family and let $p(\theta)$ be its
corresponding natural exponential family conjugate prior, writing
\begin{align}
    p(\theta) &= \exp \left\{ \langle \prioreta_\theta, t_\theta(\theta) \rangle - \log Z_\theta(\prioreta_\theta) \right\},
    \label{eq:svae_densities_start}
    \\
    p(x \given \theta) &= \exp \left\{ \langle \prioreta_{x}(\theta), t_{x}(x) \rangle - \log Z_{x}(\prioreta_{x}(\theta)) \right\}
    = \exp \left\{ \langle t_\theta(\theta), (t_{x}(x), 1) \rangle \right\},
    \label{eq:svae_densities_end}
\end{align}
where we used exponential family conjugacy to write
$t_\theta(\theta) = \left( \prioreta_{x}(\theta), -\log Z_{x}(\prioreta_{x}(\theta))
\right)$.
The local latent variables $x$ could have additional structure, like including
both discrete and continuous latent variables or tractable graph structure, but
here we keep the notation simple.

Next, we define a general likelihood function. Let $p(y \given x, \obsparam)$
be a general family of densities and let~$p(\obsparam)$ be an exponential
family prior on its parameters.
For example, each observation $y_n$ may depend on the latent value $x_n$
through an MLP, as in the density network model of Section~\ref{sec:models}.
This generic non-conjugate observation model provides modeling flexibility, yet
the SVAE can still leverage conjugate exponential family structure in inference, as
we show next.

\subsection{Stochastic variational inference algorithm}
Though the general observation model $p(y \given x, \obsparam)$ means that
conjugate updates and natural gradient SVI \citep{hoffman2013stochastic} cannot
be directly applied, we show that by generalizing the recognition network idea
we can still approximately optimize out the local variational factors
leveraging conjugacy structure.

For fixed $y$, consider the mean field family $q(\theta)q(\obsparam)q(x)$ and
the variational inference objective
\begin{equation}
    \L[ \, q(\theta) q(\obsparam) q(x) \, ]
    \triangleq \E_{q(\theta)q(\obsparam)q(x)} \! \left[ \log \frac{p(\theta)p(\obsparam)p(x \given \theta) p(y \given x, \obsparam)}{q(\theta)q(\obsparam)q(x)} \right].
    \label{eq:svae_mean_field_objective}
\end{equation}
Without loss of generality we can take the global factor $q(\theta)$ to be in
the same exponential family as the prior $p(\theta)$, and we denote its natural
parameters by $\vareta_\theta$.
We restrict $q(\obsparam)$ to be in the same exponential family as
$p(\obsparam)$ with natural parameters $\vareta_\obsparam$.
Finally, we restrict
$q(x)$ to be in the same exponential family as $p(x \given
\theta)$, writing its natural parameter as $\vareta_x$.
Using these explicit variational parameters, we write the mean field
variational inference objective in Eq.~\eqref{eq:svae_mean_field_objective} as
$\L(\vareta_\theta, \vareta_\obsparam, \vareta_x)$.

To perform efficient optimization of the objective $\L(\vareta_\theta,
\vareta_\obsparam, \vareta_x)$, we consider choosing the
variational parameter $\vareta_x$ as a function of the other
parameters $\vareta_\theta$ and $\vareta_\obsparam$.
One natural choice is to set $\vareta_x$ to be a local partial
optimizer of $\L$.
However, without conjugacy structure finding a local partial optimizer may be
computationally expensive for general densities $p(y \given x, \obsparam)$, and
in the large data setting this expensive optimization would have to be
performed for each stochastic gradient update.
Instead, we choose $\vareta_x$ by optimizing over a surrogate
objective $\widehat \L$ with conjugacy structure, given by
\begin{equation}
    \widehat \L(\vareta_\theta, \vareta_x, \phi)
    \triangleq \E_{q(\theta)q(x)} \! \left[ \log \frac{p(\theta)p(x \given \theta) \exp \{ \psi(x ; y, \phi) \} }{q(\theta) q(x)} \right],
    \quad
    \psi(x ; y, \phi) \triangleq \langle r(y ; \phi), \; t_x(x) \rangle,
    \notag
\end{equation}
where $\{ r(y ; \phi) \}_{\phi \in \R^m}$ is some parameterized class of
functions that serves as the recognition model.
Note that the potentials $\psi(x ; y, \phi)$ have a form conjugate to the exponential
family $p(x \given \theta)$.
% We call $r(y; \phi)$ the \emph{recognition model}.
We define~$\vareta^*_x(\vareta_\theta, \phi)$ to be a local
partial optimizer of $\widehat \L$ along with the
corresponding factor~$\opt q(x)$,
\begin{equation}
    \vareta^*_x(\vareta_\theta, \phi) \triangleq \argmin_{\vareta_x} \widehat \L (\vareta_\theta, \vareta_x, \phi),
    \qquad
    \opt q(x) = \exp \left\{ \langle \vareta^*_x(\vareta_\theta, \phi), \, t_x(x) \rangle - \log Z_x(\vareta_x^*(\vareta_\theta, \phi))  \right\}.
    \notag
\end{equation}
As with the variational autoencoder of Section~\ref{sec:vae_inference}, the
resulting variational factor~$\opt q(x)$ is suboptimal for the variational
objective $\L$.
However, because the surrogate objective has the same form as a variational
inference objective for a conjugate observation model, the factor~$\opt q(x)$ not only is
easy to compute but also inherits exponential family and graphical
model structure for tractable inference.

Given this choice of $\vareta^*_x(\vareta_\theta, \phi)$, the
SVAE objective is
$\Lsvae(\vareta_\theta, \vareta_\obsparam, \phi) \triangleq \L(\vareta_\theta, \vareta_\obsparam, \vareta^*_x(\vareta_\theta, \phi))$.
This objective is a lower bound for the variational inference objective
Eq.~\eqref{eq:svae_mean_field_objective} in the following sense.

\begin{proposition}[The SVAE objective lower-bounds the mean field objective]
\label{prop:svae_lower_bound}
    The SVAE objective function $\Lsvae$ lower-bounds the mean field objective
    $\L$ in the sense that
    \begin{equation}
        \max_{q(x)} \L [ \, q(\theta) q(\obsparam) q(x) \, ]
        \geq
        \max_{\vareta_x} \L(\vareta_\theta, \vareta_\obsparam, \vareta_x)
        \geq 
        \Lsvae(\vareta_\theta, \vareta_\obsparam, \phi)
        \quad
        \forall \phi \in \R^m,
    \end{equation}
    for any parameterized function class $\{r(y; \phi)\}_{\phi \in \R^m}$.
    Furthermore, if there is some $\phi^* \in \R^m$ such that
    $\psi(x ; y, \phi^*) = \E_{q(\obsparam)} \log p(y \given x, \obsparam)$,
    then the bound can be made tight in the sense that
    \begin{equation}
        \max_{q(x)} \L [ \, q(\theta) q(\obsparam) q(x) \, ]
        =
        \max_{\vareta_x} \L(\vareta_\theta, \vareta_\obsparam, \vareta_x)
        =
        \max_\phi \Lsvae(\vareta_\theta, \vareta_\obsparam, \phi).
        % =
        % \max_\phi
        % \L(\vareta_\theta, \vareta_\obsparam,
        % \vareta^*_x(\vareta_\theta, \phi)).
    \end{equation}
\end{proposition}

Thus by using gradient-based optimization to maximize
$\Lsvae(\vareta_\theta, \vareta_\obsparam, \phi)$ we are
maximizing a lower bound on the model log evidence $\log p(y)$.
In particular, by optimizing over $\phi$ we are effectively learning how to
condition on observations so as to best approximate the posterior while
maintaining conjugacy structure.
Furthermore, to provide the best lower bound we may choose the recognition
model function class $\{ r(y; \phi) \}_{\phi \in \R^m}$ to be as rich as
possible.

Choosing $\vareta^*_x(\vareta_\theta, \phi)$ to be a local
partial optimizer of $\widehat \L$ provides two computational advantages.
First, it allows $\vareta^*_x(\vareta_\theta, \phi)$ and
expectations with respect to $\opt q(x)$ to be computed efficiently by
exploiting exponential family graphical model structure.
Second, it provides computationally efficient ways to estimate the natural
gradient with respect to the latent model parameters, as we summarize next.

\begin{proposition}[Natural gradient of the SVAE objective]
\label{prop:svae_natural_gradient}
    The natural gradient of the SVAE objective $\Lsvae$ with
    respect to $\vareta_\theta$ can be estimated as
    \begin{equation}
        \natnabla_{\vareta_\theta} \Lsvae(\vareta_\theta, \vareta_\obsparam, \phi)
        =
        \left(\prioreta_\theta + \E_{\opt q(x)} \left[ (t_x(x), 1) \right] - \vareta_\theta \right)
        +
        {(\nabla^2 \log Z_\theta(\vareta_\theta))}^{-1} \nabla F(\vareta_\theta),
        \label{eq:svae_natural_gradient}
    \end{equation}
    where $F(\vareta_\theta^\prime) = \L(\vareta_\theta, \vareta_\gamma,
    \vareta_x^*(\vareta_\theta^\prime, \phi))$.
    When there is only one local variational factor $q(x)$, then we can
    simplify the estimator to
    \begin{equation}
        \natnabla_{\vareta_\theta} \Lsvae(\vareta_\theta, \vareta_\obsparam, \phi)
        =
        \left(\prioreta_\theta + \E_{\opt q(x)} \left[ (t_x(x), 1) \right] - \vareta_\theta \right)
        +
        (\nabla_{\vareta_x} \L(\vareta_\theta, \vareta_\obsparam, \vareta^*_x(\vareta_\theta, \phi)), 0).
        \label{eq:svae_natural_gradient_2}
    \end{equation}
\end{proposition}

Note that the first term in Eq.~\eqref{eq:svae_natural_gradient} is the same as
the expression for the natural gradient in SVI for conjugate models
\citep{hoffman2013stochastic}, while a stochastic estimate of $\nabla
F(\vareta_\theta)$ in the first expression or, alternatively, a stochastic
estimate of $\nabla_{\vareta_\theta} \L(\vareta_\theta, \vareta_\gamma,
\vareta_x^*(\vareta_\theta, \phi))$ in the second expression
is computed automatically as part of the backward pass for computing the
gradients with respect to the other parameters, as described next.
Thus we have an expression for the natural gradient with respect to the latent
model's parameters that is almost as simple as the one for conjugate models,
differing only by a term involving the neural network likelihood function.
Natural gradients are invariant to smooth invertible reparameterizations of the
variational family \citep{amari1998natural,amari2007methods} and provide
effective second-order optimization updates
\citep{martens2015optimizing,martens2015perspectives}.

The gradients of the objective with respect to the other variational
parameters, namely $\nabla_{\vareta_\obsparam}
\Lsvae(\vareta_\theta, \vareta_\obsparam, \phi)$ and~$\nabla_{\phi} \Lsvae(\vareta_\theta, \vareta_\obsparam, \phi)$,
can be computed using the reparameterization trick and standard automatic
differentiation techniques.
To isolate the terms that require the reparameterization trick, we rearrange
the objective as
\begin{equation}
  \Lsvae(\vareta_\theta, \vareta_\obsparam, \phi) = \E_{q(\obsparam) \opt q(x)} \log p(y \given x, \obsparam) - \KL(q(\theta) \opt q(x) \, \| \, p(\theta, x)) - \KL(q(\obsparam) \,\|\, p(\obsparam)).
    \label{eq:svae_objective_with_kl}
\end{equation}
The KL divergence terms are between members of the same tractable
exponential families.
An unbiased estimate of the first term can be computed by sampling $\hat
x \sim \opt q(x)$ and $\hat \obsparam \sim q(\obsparam)$ and computing
$\nabla_{\vareta_\obsparam, \phi} \log p(y \given \hat x, \hat
\obsparam)$ with automatic differentiation.

\section{Related work}
In addition to the papers already referenced, there are several recent papers to
which this work is related.

The two papers closest to this work are \citet{krishnan2015deep} and
\citet{archer2015black}.
In \citet{krishnan2015deep} the authors consider combining variational
autoencoders with continuous state-space models, emphasizing the relationship
to linear dynamical systems (also called Kalman filter models).
They primarily focus on nonlinear dynamics and an RNN-based variational family,
as well as allowing control inputs.
However, the approach does not extend to general graphical models or discrete
latent variables.
It also does not leverage natural gradients or exact inference subroutines.

In \citet{archer2015black} the authors also consider the problem of variational
inference in general continuous state space models but focus on using a
structured Gaussian variational family without considering parameter learning.
As with \citet{krishnan2015deep}, this approach does not include discrete latent
variables (or any latent variables other than the continuous states).
However, the method they develop could be used with an SVAE to handle inference
with nonlinear dynamics.

In addition, both \citet{gregor2015draw} and \citet{chung2015recurrent}
extend the variational autoencoder framework to sequential models, though they
focus on RNNs rather than probabilistic graphical models.
% More classical approaches for nonlinear filtering, smoothing, and model fitting
% are surveyed in \citet{krishnan2015deep}.
\todo{cite deep sigmoid}
\todo{cite embed2control}
\todo{cite percy 2008}
\todo{cite sontag dkf}

Finally, there is much related work on handling nonconjugate model terms in mean
field variational inference.
In \citet{khan2015kullback} and \citet{khan2016faster} the authors present a
general scheme that is able to exploit conjugate exponential family
structure while also handling arbitrary nonconjugate model factors, including
the nonconjugate observation models we consider here.
In particular, they propose using a proximal gradient framework and splitting
the variational inference objective into a difficult term to be linearized
(with respect to mean parameters) and a tractable concave term, so that the
resulting proximal gradient update is easy to compute, just like in a fully
conjugate model.
In \citet{knowles2011ncvmp}, the authors propose performing natural gradient
descent with respect to natural parameters on each of the variational factors
in turn, and they focus on approximating expectations of nonconjugate energy
terms in the objective with model-specific lower-bounds (rather than estimating
them with generic Monte Carlo).
As in conjugate SVI \citep{hoffman2013stochastic}, they observe that, on
conjugate factors and with an undamped update (i.e.~a unit step size), the
natural gradient update reduces to the standard conjugate mean field update.

In contrast to the approaches of \citet{khan2015kullback},
\citet{khan2016faster}, and \citet{knowles2011ncvmp}, rather than linearizing
intractable terms around the current iterate, in this work we handle
intractable terms via recognition networks and amoritized inference (and the
remaining tractable objective terms are multi-concave in general, analogous to
SVI \citep{hoffman2013stochastic}).
That is, we use parametric function approximators to learn to condition on
evidence in a conjugate form.
We expect these approaches to handling nonconjugate objective terms
may be complementary, and the best choice may be situation-dependent.
For models with local latent variables and datasets where minibatch-based
updating is important, using inference networks to compute local variational
parameters in a fixed-depth circuit (as in the VAE
\citep{kingma2013autoencoding,rezende2014stochastic}) or optimizing out the
local variational factors using fast conjugate updates (as in conjugate SVI
\citep{hoffman2013stochastic}) can be advantageous because in both cases local
variational parameters for the entire dataset need not be maintained across
updates.
The SVAE we propose here is a way to combine the inference network and
conjugate SVI approaches.

\section{Experiments}
\label{sec:experiments}

We apply the SVAE to both synthetic and real data and demonstrate its ability
to learn feature representations and latent structure.
Code is available at \href{https://github.com/mattjj/svae}{\texttt{github.com/mattjj/svae}}.
% See also videos of
% \href{https://gfycat.com/ComplicatedMenacingHagfish}{training a warped mixture}
% and
% \href{https://gfycat.com/GoodnaturedDampBluebreastedkookaburra}{training a nonlinear LDS}.

\subsection{LDS SVAE for modeling synthetic data}
Consider a sequence of 1D images representing a dot bouncing from one side of
the image to the other, as shown at the top of Fig.~\ref{fig:dots}.
We use an LDS SVAE to find a low-dimensional latent state space
representation along with a nonlinear image model.
The model is able to represent the image accurately and to make long-term
predictions with uncertainty. See supplementals for details.

\begin{figure*}[tb]
  \centering
  \begin{subfigure}{0.475\textwidth}
    \centering
    \includegraphics[width=\textwidth]{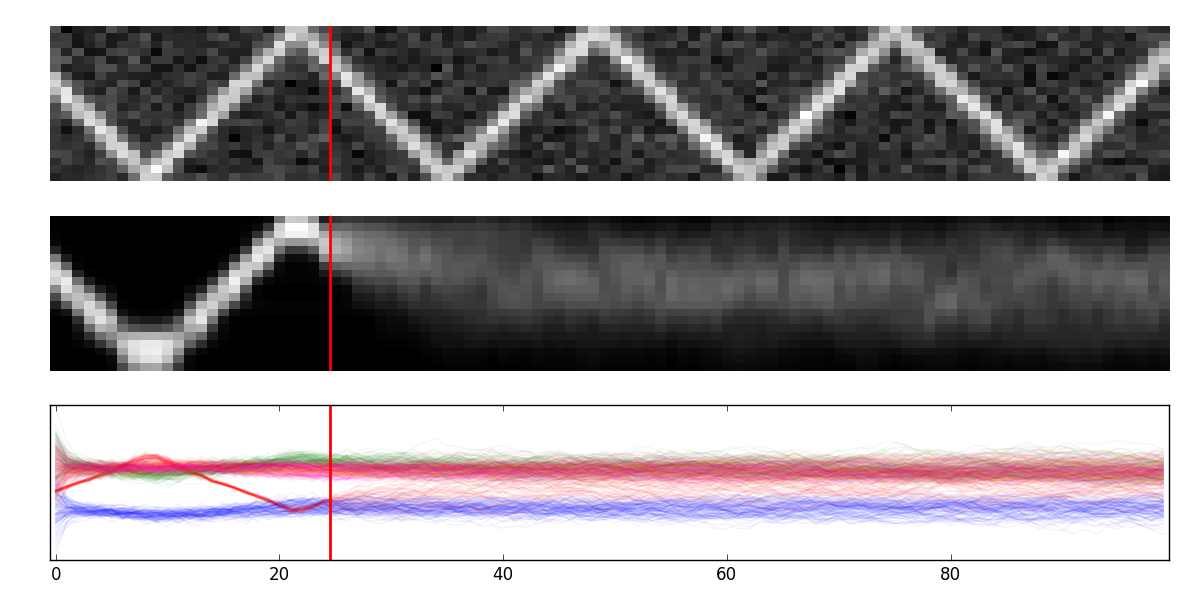}
    \caption{Predictions after 200 training steps.}
    \label{fig:dot_1}
  \end{subfigure}
  \begin{subfigure}{0.475\textwidth}
    \centering
    \includegraphics[width=\textwidth]{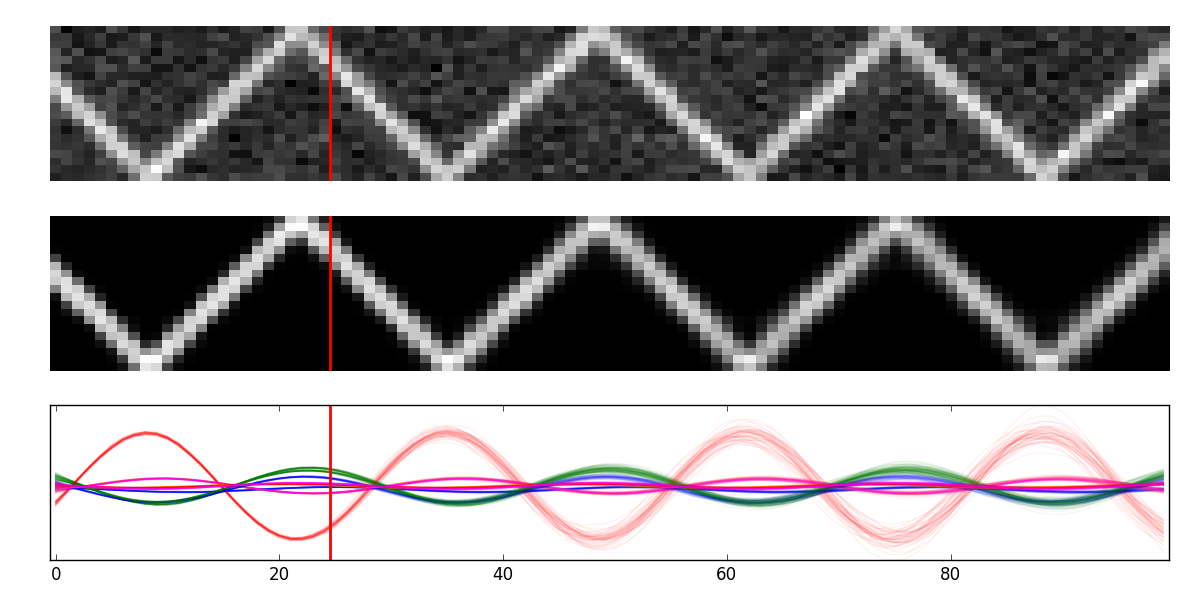}
    \caption{Predictions after 1100 training steps.}
    \label{fig:dot_2}
  \end{subfigure}
  \caption{
  Predictions from an LDS SVAE fit to 1D dot image data at two stages of training.
  The top panel shows an
  example sequence with time on the horizontal axis. The middle panel
  shows the noiseless predictions given data up to the vertical
  line, while the bottom panel shows the latent states.
}
  \label{fig:dots}
\end{figure*}

\begin{figure}
  \centering
  \begin{subfigure}[t]{0.55\textwidth}
    \centering
    \includegraphics[width=0.8\textwidth, clip, trim=0cm 0.5cm 0.1cm 0.1cm]{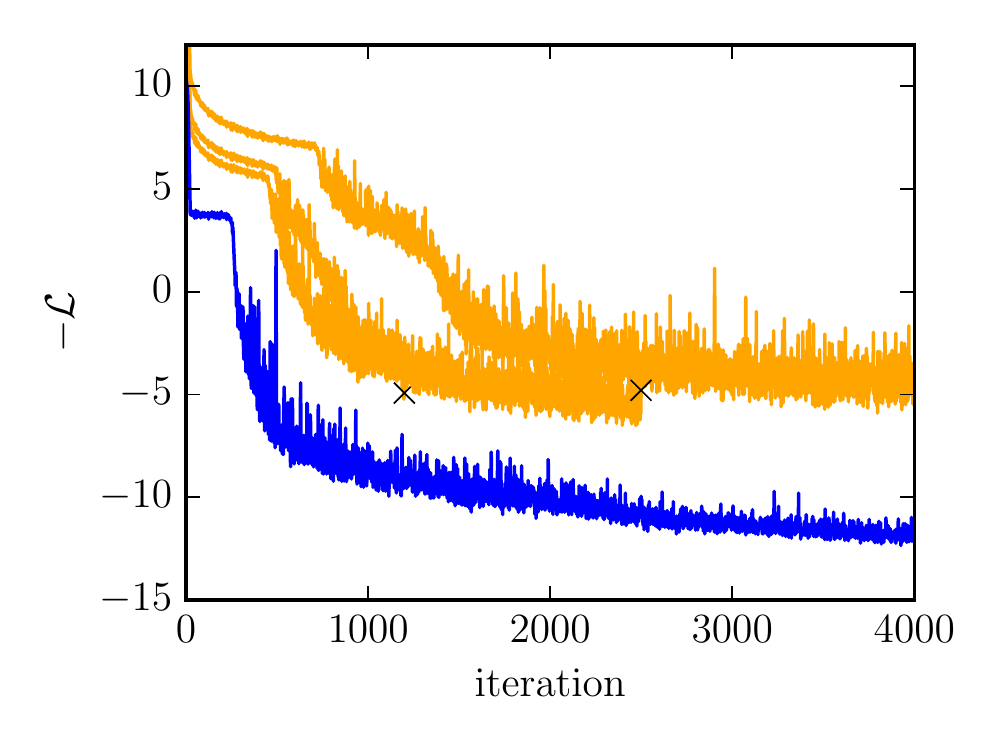}
    \caption{Natural (blue) and standard (orange) gradient updates.}
    \label{fig:dot_gradients}
  \end{subfigure}
  \begin{subfigure}[t]{0.43\textwidth}
    \centering
    \includegraphics[width=0.7\textwidth]{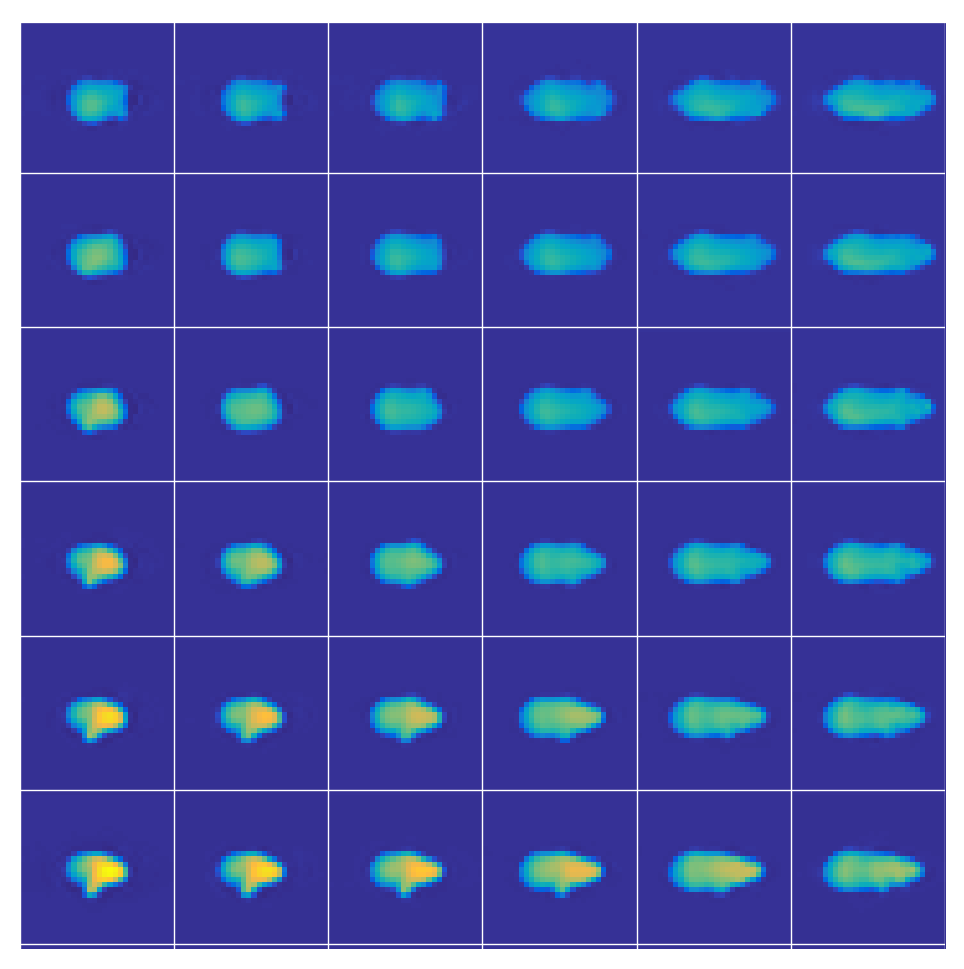}
    \caption{Subspace of learned observation model.}
    \label{fig:mouse_vae}
  \end{subfigure}
  \caption{Experimental results from LDS SVAE models on synthetic data and real mouse data.}
  % \caption{
  %   Panel~\subref{fig:dot_gradients} compares natural and standard gradient
  %   updates and Panel~\subref{fig:mouse_vae} shows a random 2D subspace in the
  %   image manifold coordinates learned by fitting a VAE to mouse depth video data.}
\end{figure}

This experiment also demonstrates the optimization advantages that can be
provided by the natural gradient updates.
In Fig.~\ref{fig:dot_gradients} we compare natural gradient updates with
standard gradient updates at three different learning rates.
The natural gradient algorithm not only learns much faster but also is less
dependent on parameterization details: while the natural gradient update used
an untuned stepsize of 0.1, the standard gradient dynamics at step sizes of
both 0.1 and 0.05 resulted in some matrix parameters to be updated to
indefinite values.
% This particular error with standard gradients can be avoided by working in a
% different parameterization, but we used a consistent parameterization to make
% the learning rate comparison more direct.

\subsection{LDS SVAE for modeling video}

We also apply an LDS SVAE to model depth video recordings of mouse behavior.
We use the dataset from \citet{wiltschko2015mapping} in which a
mouse is recorded from above using a Microsoft Kinect.
We used a subset consisting of 8 recordings, each of a distinct
mouse, 20 minutes long at 30 frames per second, for a total of
288000 video fames downsampled to $30 \times 30$ pixels.

We use MLP observation and recognition models with two hidden layers of
200 units each and a 10D latent space.
Fig.~\ref{fig:mouse_vae} shows images corresponding to a regular grid on a
random 2D subspace of the latent space, illustrating that the learned image
manifold accurately captures smooth variation in the mouse's body pose.
Fig.~\ref{fig:svae_lds_predictions} shows predictions from the model paired
with real data.

\begin{figure}[tb]
  \centering
  \begin{subfigure}{\textwidth}
    \centering
    \includegraphics[width=\textwidth]{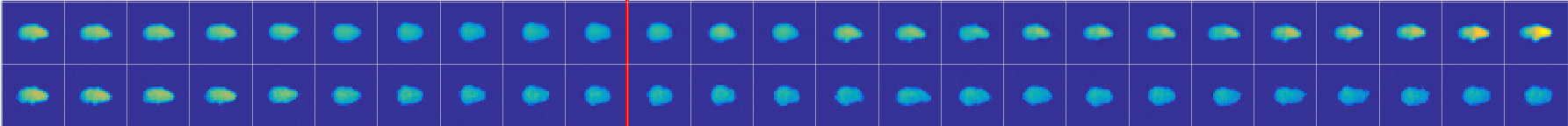}
  \end{subfigure}
  \\
  \begin{subfigure}{\textwidth}
    \centering
    \includegraphics[width=\textwidth]{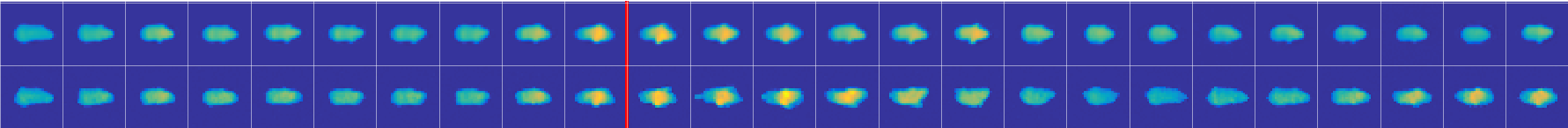}
  \end{subfigure}
  \\
  \begin{subfigure}{\textwidth}
    \centering
    \includegraphics[width=\textwidth]{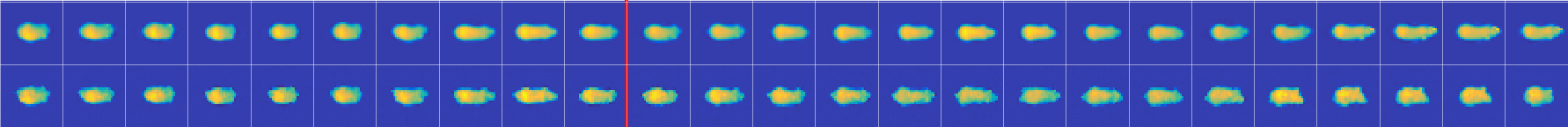}
  \end{subfigure}
  \\
  \caption{
  Predictions from an LDS SVAE fit to depth video.
  In each panel, the top is a sampled prediction and the
  bottom is real data. The model is conditioned on observations to the left of the line.
  }
  \label{fig:svae_lds_predictions}
\end{figure}

\subsection{SLDS SVAE for parsing behavior}

Finally, because the LDS SVAE can accurately represent the depth video over
short timescales, we apply the latent switching linear dynamical system (SLDS)
model to discover the natural units of behavior.
Fig.~\ref{fig:slds_svae_syllables} and Fig.~\ref{fig:slds_svae_syllables_2} in
the appendix show some of the discrete states that arise from fitting an SLDS
SVAE with 30 discrete states to the depth video data.
The discrete states that emerge show a natural clustering of short-timescale
patterns into behavioral units. See the supplementals for more.

\begin{figure}[tb]
  \centering
  \begin{subfigure}{0.85\textwidth}
    \centering
    \includegraphics[width=\textwidth]{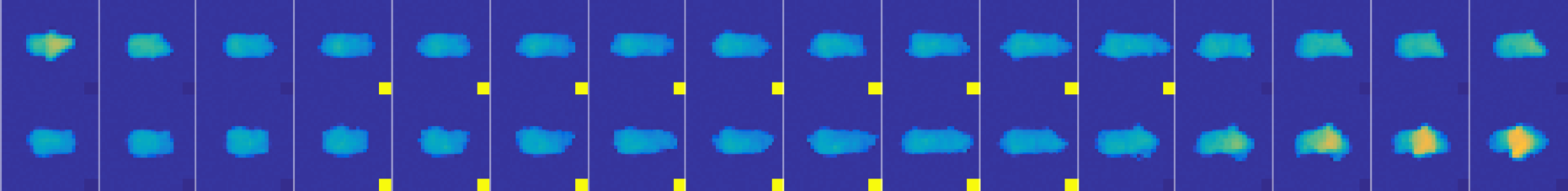}
    \caption{Extension into running}
  \end{subfigure}
  \\
  \begin{subfigure}{0.85\textwidth}
    \centering
    \includegraphics[width=\textwidth]{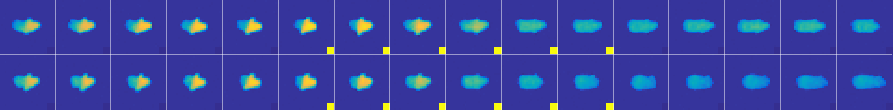}
    \caption{Fall from rear}
  \end{subfigure}
  \caption{Examples of behavior states inferred from depth video.
    Each frame sequence is padded on both sides, with a square in the
  lower-right of a frame depicting when the state is the most probable.}
  \label{fig:slds_svae_syllables}
\end{figure}

\section{Conclusion}
Structured variational autoencoders provide a general framework that combines
some of the strengths of probabilistic graphical models and deep learning
methods. In particular, they use graphical models both to give models rich
latent representations and to enable fast variational inference with CRF-like
structured approximating distributions. To complement these structured
representations, SVAEs use neural networks to produce not only flexible
nonlinear observation models but also fast recognition networks that map
observations to conjugate graphical model potentials.

% We focused on exponential family graphical models for both the generative and variational latent models for their algorithmic tractability.
% However, this framework could be extended to use more flexible generative models, such as neural network state space models, while still using structured latent variational models at the cost of requiring black-box inference methods \citep{ranganath:2014-bbvi}.

% TODO mention hierarchical modeling

\clearpage
\printbibliography%

\clearpage
\appendix

\section{Optimization}

In this section we fix our notation for gradients and establish some basic
definitions and results that we use in the sequel.

\subsection{Gradient notation}

We follow the notation in \citet[A.5]{bertsekas1999nonlinear}.
In particular, if $f: \R^n \to \R^m$ is a continuously differentiable function,
we define the gradient matrix of $f$, denoted $\nabla f(x)$, to be the $n
\times m$ matrix in which the $i$th column is the gradient $\nabla f_i(x)$ of
$f_i$, the $i$th coordinate function of $f$, for $i=1,2,\ldots,m$.
That is,
\begin{align}
    \nabla f(x) = \begin{bmatrix} \nabla f_1(x) & \cdots & \nabla f_m(x) \end{bmatrix}.
\end{align}
The transpose of $\nabla f$ is the Jacobian matrix of $f$, in which the $ij$th
entry is the function $\partial f_i / \partial x_j$.

If $f: \R^n \to \R$ is continuously differentiable with continuously
differentiable partial derivatives, then we define the Hessian matrix of $f$,
denoted $\nabla^2 f$, to be the matrix in which the $ij$th entry is the
function $\partial^2 f / \partial x_i \partial x_j$.

Finally, if $f: \R^n \times \R^m \to \R$ is a function of $(x, y)$ with $x \in
R^n$ and $y \in \R^m$, we write
\begin{gather}
\nabla_x f(x, y) = \begin{pmatrix} \frac{\partial f(x, y)}{\partial x_1} \\ \vdots \\ \frac{\partial f(x, y)}{\partial x_m} \end{pmatrix},
    \qquad
    \nabla_y f(x, y) = \begin{pmatrix} \frac{\partial f(x,y)}{\partial y_1} \\ \vdots \\ \frac{\partial f(x,y)}{\partial y_n} \end{pmatrix}
    \\
    \nabla_{xx}^2 f(x,y) = \left( \frac{\partial^2 f(x,y)}{\partial x_i \partial x_j} \right),
    \qquad
    \nabla_{yy}^2 f(x,y) = \left( \frac{\partial^2 f(x,y)}{\partial y_i \partial y_j} \right),
    \\
    \nabla_{xy}^2 f(x,y) = \left( \frac{\partial^2 f(x,y)}{\partial x_i \partial y_j} \right).
\end{gather}

\subsection{Local and partial optimizers}

In this section we state the definitions of local partial optimizer and
necessary conditions for optimality that we use in the sequel.

% \begin{definition}[Local optimizer]
%     Let $f: \R^n \to \R$ be an objective function to be maximized. We call a
%     point $\opt x \in \R^n$ a \emph{local optimizer} of $f$ if there exists an
%     $\epsilon > 0$ such that
%     \begin{align}
%         f(x) \leq f(\opt x) \quad \forall \, x \textup{ with } \| x - \opt x \| < \epsilon,
%     \end{align}
%     where $\| \, \cdot \, \|$ is any vector norm.
% \end{definition}

\begin{definition}[Partial optimizer, local partial optimizer]
    Let $f: \R^n \times \R^m \to \R$ be an objective function to be maximized.
    For a fixed $x \in \R^n$, we call a point $\opt y \in \R^m$ an
    \emph{unconstrained partial optimizer} of $f$ given $x$ if
    \begin{align}
        f(x, y) \leq f(x, \opt y) \quad \forall \, y \in \R^m
    \end{align}
    and we call $\opt y$ an \emph{unconstrained local partial optimizer} of $f$
    given $x$ if there exists an $\epsilon > 0$ such that
    \begin{align}
        f(x, y) \leq f(x, \opt y) \quad \forall \, y \textup{ with } \| y - \opt y \| < \epsilon,
    \end{align}
    where $\| \, \cdot \, \|$ is any vector norm.
\end{definition}

\begin{proposition}[Necessary conditions for optimality, Prop.~3.1.1 of \citet{bertsekas1999nonlinear}]
    Let $f: \R^n \times \R^m \to \R$ be continuously differentiable.
    For fixed $x \in \R^n$ if $\opt y \in \R^m$ is an unconstrained local
    partial optimizer for $f$ given $x$ then
    \begin{align}
        \nabla_y f(x, \opt y) = 0.
    \end{align}
    If instead $x$ and $y$ are subject to the constraints $h(x, y) = 0$
    for some continuously differentiable $h: \R^n \times \R^m \to \R^m$
    and $\opt y$ is a constrained local partial optimizer for $f$ given $x$
    with the regularity condition that $\nabla_y h(x, \opt y)$ is full rank,
    then there exists a Lagrange multiplier $\opt \lambda \in \R^m$ such that
    \begin{align}
        \nabla_y f(x, \opt y) + \nabla_y h(x, \opt y) \opt \lambda
        =
        0,
    \end{align}
    and hence the cost gradient $\nabla_y f(x, \opt y)$ is orthogonal to the
    first-order feasible variations in $y$ given by the null space of $\nabla_y
    h(x, \opt y)^\T$.
\end{proposition}

Note that the regularity condition on the constraints is not needed if the
constraints are linear \citep[Prop.~3.3.7]{bertsekas1999nonlinear}.

For a continuously differentiable function $f: \R^n \to \R$, we say $\opt x$ is
a stationary point of $f$ if $\nabla f(\opt x) = 0$.
For general unconstrained smooth optimization, the limit points of
gradient-based algorithms are guaranteed only to be stationary points of the
objective, not necessarily local optima.
Block coordinate ascent methods, when available, provide slightly stronger
guarantees: not only is every limit point a stationary point of the objective,
in addition each coordinate block is a partial optimizer of the objective.
Note that the objective functions we consider maximizing in the following are
bounded above.

\subsection{Partial optimization and the Implicit Function Theorem}
\label{sec:ift}

Let $f: \R^n \times \R^m \to \R$ be a scalar-valued objective function of two
unconstrained arguments $x \in \R^n$ and $y \in \R^m$, and let $\opt y: \R^n
\to \R^m$ be some function that assigns to each $x \in \R^n$ a value $\opt y(x)
\in \R^m$.
Define the composite function $g: \R^n \to \R$ as
\begin{align}
    g(x) \triangleq f(x, \opt y(x))
\end{align}
and using the chain rule write its gradient as
\begin{align}
    \nabla g(x) &= \nabla_x f(x, \opt y(x)) +  \nabla \opt y(x) \nabla_y f(x, \opt y(x)).
    \label{eq:g_grad}
\end{align}

One choice of the function $\opt y(x)$ is to partially optimize $f$ for any
fixed value of $x$.
For example, assuming that $\argmax_y f(x, y)$ is nonempty for every $x \in
\R^n$, we could choose $\opt y$ to satisfy $\opt y(x) \in \argmax_y f(x, y)$, so that
$g(x) = \max_y f(x,y)$.\footnote{%
For a discussion of differentiability issues when there is more than one
optimizer, i.e.~when $\argmax_y f(x,y)$ has more than one element, see
\citet{danskin1967theory}, \citet[Section 2.4]{fiacco1984introduction}, and
\citet[Chapter 4]{bonnans2000perturbation}.
Here we only consider the sensitivity of local stationary points and assume
differentiability almost everywhere.
}
Similarly, if $\opt y(x)$ is chosen so that $\nabla_y f(x, \opt y(x)) = 0$, which is satisfied
when $\opt y(x)$ is an unconstrained local partial optimizer for $f$ given $x$, then the
expression in Eq.~\eqref{eq:g_grad} can be simplified as in the following
proposition.

\begin{proposition}[Gradients of locally partially optimized objectives]
\label{prop:gradients_of_locally_partial_opt}
    Let $f: \R^n \times \R^m \to \R$ be continuously differentiable, let $\opt
    y$ be a local partial optimizer of $f$ given $x$ such that $y^*(x)$ is
    differentiable, and define $g(x) = f(x, \opt y(x))$.
    Then
    \begin{align}
        \nabla g(x) = \nabla_x f(x, \opt y(x)).
    \end{align}
\end{proposition}
\begin{proof}
    If $\opt y$ is an unconstrained local partial optimizer of $f$ given $x$ then it
    satisfies $\nabla_y f(x, \opt y) = 0$, and if $\opt y$ is a
    regularly-constrained local partial optimizer then the feasible variation $\nabla
    \opt y(x)$ is orthogonal to the cost gradient $\nabla_y f(x, \opt y)$. In
    both cases the second term in the expression for $\nabla g(x)$ in
    Eq.~\eqref{eq:g_grad} is zero.
\end{proof}

In general, when $\opt y(x)$ is not a stationary point of $f(x, \cdot)$,
to evaluate the gradient $\nabla g(x)$ we need to evaluate $\nabla \opt y(x)$
in Eq.~\eqref{eq:g_grad}.
However, this term may be difficult to compute directly.
The function $\opt y(x)$ may arise implicitly from some system of equations of the
form $h(x, y) = 0$ for some continuously differentiable function $h: \R^n
\times \R^m \to \R^m$.
For example, the value of $y$ may be computed from $x$ and $h$ using a
black-box iterative numerical algorithm.
However, the Implicit Function Theorem provides another means to compute
$\nabla \opt y(x)$ using only the derivatives of $h$ and the value of $\opt y(x)$.

\begin{proposition}[Implicit Function Theorem, Prop.~A.25 of \citet{bertsekas1999nonlinear}]
\label{prop:ift}
    Let $h: \R^n \times \R^m \to \R^m$ be a function and $\bar{x} \in \R^n$ and
    $\bar{y} \in \R^m$ be points such that
    \begin{enumerate}
        \item $h(\bar{x}, \bar{y}) = 0$
        \item $h$ is continuous and has a continuous nonsingular gradient
            matrix $\nabla_y h(x, y)$ in an open set containing $(\bar{x},
            \bar{y})$.
    \end{enumerate}
    Then there exist open sets $S_{\bar{x}} \subseteq \R^n$ and $S_{\bar{y}}
    \subseteq \R^m$ containing $\bar{x}$ and $\bar{y}$, respectively, and a
    continuous function $\opt y: S_{\bar{x}} \to S_{\bar{y}}$ such that $\bar{y} =
    \opt y(x)$ and $h(x, \opt y(x)) = 0$ for all $x \in S_{\bar{x}}$.
    The function $\opt y$ is unique in the sense that if $x \in S_{\bar{x}}$, $y
    \in S_{\bar{y}}$, and $h(x, y) = 0$, then $y = \opt y(x)$. Furthermore, if for
    some $p > 0$, $h$ is $p$ times continuously differentiable, the same is
    true for $\opt y$, and we have
    \begin{align}
        \nabla \opt y(x) = - \nabla_x h \left(x, \opt y(x) \right) \left( \nabla_y h
        \left( x, \opt y(x) \right) \right)^{-1}, \qquad \forall \; x \in
        S_{\bar{x}}.
    \end{align}
\end{proposition}

As a special case, the equations $h(x,y) = 0$ may be the first-order
stationary conditions of another unconstrained optimization problem.
That is, the value of $y$ may be chosen by locally partially optimizing the
value of $u(x,y)$ for a function $u: \R^n \times \R^m \to \R$ with no
constraints on $y$, leading to the following corollary.

\begin{corollary}[Implicit Function Theorem for optimization subroutines]
\label{cor:implicit_function_theorem_for_optimization}
    Let $u : \R^n \times \R^m \to \R$ be a twice continuously differentiable function
    such that the choice $h = \nabla_y u$ satisfies the hypotheses of
    Proposition~\ref{prop:ift} at some point $(\bar{x}, \bar{y})$, and define
    $\opt y$ as in Proposition~\ref{prop:ift}.
    Then we have
    \begin{align}
        \nabla \opt y(x) = -\nabla^2_{xy} u \left( x, \opt y(x) \right) \left( \nabla^2_{yy} u
        \left( x, \opt y(x) \right) \right)^{-1}, \qquad \forall \; x \in
        S_{\bar{x}}.
    \end{align}
\end{corollary}

\section{Exponential families}
\label{sec:exp_fam}

In this section we set up notation for exponential families and outline some
basic results.
Throughout this section we take all densities to be absolutely continuous with
respect to the appropriate Lebesgue measure (when the underlying set $\X$ is
Euclidean space) or counting measure (when $\X$ is discrete), and
denote the Borel $\sigma$-algebra of a set $\X$ as
$\B(\X)$ (generated by Euclidean and discrete topologies,
respectively).
We assume measurability of all functions as necessary.

Given a statistic function $t_x: \X \to \R^n$ and a base measure
$\nu_\X$, we can define an exponential family of probability densities on
$\X$ relative to $\nu_\X$ and indexed by natural parameter $\sectwoeta_x
\in \R^n$ by
\begin{align}
    p(x \given \sectwoeta_x) \propto \exp \left\{ \langle \sectwoeta_x, \, t_x(x) \rangle \right\},
    \quad \forall \sectwoeta_x \in \R^n,
\end{align}
where $\langle \cdot, \, \cdot \rangle$ is the standard inner product on $\R^n$.
We also define the partition function as
\begin{align}
    Z_x(\sectwoeta_x) \triangleq \int \exp \left\{ \langle \sectwoeta_x, \, t_x(x) \rangle \right\} \nu_\X(dx)
\end{align}
and define $H \subseteq \R^n$ to be the set of all normalizable natural parameters,
\begin{align}
    H \triangleq \left\{ \sectwoeta \in \R^n : Z_x(\sectwoeta) < \infty \right\}.
\end{align}
We can write the normalized probability density as
\begin{align}
    p(x \given \sectwoeta) = \exp \left\{ \langle \sectwoeta_x, \, t_x(x) \rangle - \log Z_x(\sectwoeta_x) \right\}.
    \label{eq:natural_exp_fam}
\end{align}
We say that an exponential family is \emph{regular} if $H$ is open, and
\emph{minimal} if there is no $\sectwoeta \in \R^n \setminus \{0\}$ such that
$\langle \sectwoeta, \, t_x(x) \rangle = 0$ ($\nu_\X$-a.e.).
We assume all families are regular and minimal.%
\footnote{Families that are not minimal, like the density of the categorical
  distribution, can be treated by restricting all algebraic operations to the
  subspace spanned by the statistic, i.e.~to the smallest $V \subset \R^n$ with
  $\range t_x \subseteq V$.}
Finally, when we parameterize the family with some other coordinates $\theta$,
we write the natural parameter as a continuous function $\sectwoeta_x(\theta)$ and
write the density as
\begin{align}
    p(x \given \theta) = \exp \left\{ \langle \sectwoeta_x(\theta), \, t_x(x) \rangle - \log Z_x(\sectwoeta_x(\theta)) \right\}
\end{align}
and take $\Theta = \sectwoeta_x^{-1}(H)$ to be the open set of parameters that correspond
to normalizable densities.
We summarize this notation in the following definition.

\begin{definition}[Exponential family of densities]
\label{def:exp_fam}
    Given a measure space $(\X, \B(\X), \nu_\X)$, a statistic
    function $t_x: \X \to \R^n$, and a natural parameter function
    $\sectwoeta_x: \Theta \to \R^n$, the corresponding \emph{exponential family
    of densities} relative to $\nu_\X$ is
    \begin{align}
        p(x \given \theta) = \exp \left\{ \langle \sectwoeta_x(\theta), \, t_x(x) \rangle - \log Z_x(\sectwoeta_x(\theta)) \right\},
    \end{align}
    where
    \begin{align}
        \log Z_x(\sectwoeta_x) \triangleq \log \int \exp \left\{ \langle \sectwoeta_x, \, t_x(x)
    \rangle \right\} \nu_\X(dx)
    \end{align}
    is the log partition function.
\end{definition}

When we write exponential families of densities for different random variables,
we change the subscripts on the statistic function, natural parameter function,
and log partition function to correspond to the symbol used for the random
variable.
When the corresponding random variable is clear from context,
we drop the subscripts to simplify notation.

The next proposition shows that the log partition function of an exponential
family generates cumulants of the statistic.

\begin{proposition}[Gradients of $\log Z$ and expected statistics]
\label{prop:gradient_of_log_Z}
    The gradient of the log partition function of an exponential family gives
    the expected sufficient statistic,
    \begin{align}
        \nabla \log Z(\sectwoeta) = \E_{p(x \given \sectwoeta)} \left[ t(x) \right],
        \label{eq:grad_logZ}
    \end{align}
    where the expectation is over the random variable $x$ with density $p(x
    \given \sectwoeta)$.
    More generally, the moment generating function of $t(x)$ can be written
    \begin{align}
        \MGF_{t(x)}(s) \triangleq \E_{p(x \given \sectwoeta)} \left[ e^{\langle s, t(x) \rangle} \right] = e^{\log Z(\sectwoeta + s) - \log Z(\sectwoeta)}
    \end{align}
    and so derivatives of $\log Z$ give cumulants of $t(x)$, where the first
    cumulant is the mean and the second and third cumulants are the second and
    third central moments, respectively.
\end{proposition}
% \begin{proof}
%     To show Eq.~\eqref{eq:grad_logZ} we write
%     \begin{align}
%         \nabla \log Z(\sectwoeta)
%         &= \nabla_\sectwoeta \log \int e^{ \langle \sectwoeta , t(x) \rangle } \nu(dx)
%         \\
%         &= \frac{1}{\int e^{ \langle \sectwoeta , t(x) \rangle } \nu(dx)}
%         \int t(x) e^{ \langle \sectwoeta, t(x) \rangle } \nu(dx)
%         \\
%         &= \int t(x) p(x \given \sectwoeta) \nu(dx)
%         \\
%         &= \E_{p(x \given \sectwoeta)}\left[ \, t(x) \, \right].
%     \end{align}
%     To derive the form of the moment generating function, we write
%     \begin{align}
%         \E_{p(x \given \sectwoeta)}\left[ \, e^{ \langle s , t(x) \rangle } \, \right]
%         &= \int e^{\langle s, t(x) \rangle } p(x \given \sectwoeta) \nu(dx)
%         \\
%         &= \int e^{ \langle s, t(x) \rangle } e^{ \langle \sectwoeta, t(x) \rangle - \log Z(\sectwoeta) } \nu(dx)
%         \\
%         &= e^{ \log Z(\sectwoeta + s) - \log Z(\sectwoeta)}.
%     \end{align}
% \end{proof}

Given an exponential family of densities on $\X$ as in
Definition~\ref{def:exp_fam}, we can define a related exponential family of
densities on $\Theta$ by defining a statistic function $t_\theta(\theta)$ in
terms of the functions $\sectwoeta_x(\theta)$ and $\log Z_x(\sectwoeta_x(\theta))$.

\begin{definition}[Natural exponential family conjugate prior]
\label{def:conj_prior}
    Given the exponential family $p(x \given \theta)$ of Definition~\ref{def:exp_fam},
    define the statistic function $t_\theta: \Theta \to \R^{n+1}$ as the
    concatenation
    \begin{align}
        t_\theta(\theta) \triangleq \left( \sectwoeta_x(\theta), - \log Z_x(\sectwoeta_x(\theta)) \right),
    \end{align}
    where the first $n$ coordinates of $t_\theta(\theta)$ are given by
    $\sectwoeta_x(\theta)$ and the last coordinate is given by $-\log Z_x(\sectwoeta_x(\theta))$.
    We call the exponential family with statistic $t_\theta(\theta)$ the
    \emph{natural exponential family conjugate prior} to the density $p(x \given
    \theta)$ and write
    \begin{align}
        p(\theta) = \exp \left\{ \langle \sectwoeta_\theta, \, t_\theta(\theta) \rangle - \log Z_\theta(\sectwoeta_\theta) \right\}
    \end{align}
    where $\sectwoeta_\theta \in \R^{n+1}$ and the density is taken relative to some
    measure $\nu_\Theta$ on $(\Theta, \B(\Theta))$.
\end{definition}

Notice that using $t_\theta(\theta)$ we can rewrite the original density $p(x \given \theta)$ as
\begin{align}
    p(x \given \theta) &= \exp \left\{ \langle \sectwoeta_x(\theta), \, t_x(x) \rangle - \log Z_x(\sectwoeta_x(\theta)) \right\}
    \\
    &= \exp \left\{ \langle t_\theta(\theta), \, (t_x(x), 1) \rangle \right\}.
\end{align}
This relationship is useful in Bayesian inference: when the exponential family
$p(x \given \theta)$ is a likelihood function and the family $p(\theta)$ is
used as a prior, the pair enjoy a convenient conjugacy property, as summarized
in the next proposition.

\begin{proposition}[Conjugacy]
\label{prop:conjugate_densities}
    Let the densities $p(x \given \theta)$ and $p(\theta)$ be defined as in
    Definitions~\ref{def:exp_fam} and~\ref{def:conj_prior}, respectively.
    We have the relations
    \begin{align}
        p(\theta, x) &= \exp \left\{ \langle \sectwoeta_\theta + (t_x(x), 1), \, t_\theta(\theta) \rangle - \log Z_\theta(\sectwoeta_\theta) \right\}
        \label{eq:conjugate_joint_density}
        \\
        p(\theta \given x) &= \exp \left\{ \langle \sectwoeta_\theta + (t_x(x), 1), \, t_\theta(\theta) \rangle - \log Z_\theta(\sectwoeta_\theta + (t_x(x), 1)) \right\}
    \end{align}
    and hence in particular the posterior $p(\theta \given x)$ is in the same
    exponential family as $p(\theta)$ with the natural parameter $\sectwoeta_\theta +
    (t_x(x), 1)$.
    Similarly, with multiple likelihood terms $p(x_i \given \theta)$ for
    $i=1,2,\ldots,N$ we have
    \begin{align}
        p(\theta) \prod_{i=1}^N p(x_i \given \theta) &= \exp \left\{ \langle \sectwoeta_\theta + \sum_{i=1}^N (t_x(x_i), 1), \, t_\theta(\theta) \rangle - \log Z_\theta(\sectwoeta_\theta) \right\}.
        \label{eq:conjugacy_multiple_likelihoods}
    \end{align}
\end{proposition}

Finally, we give a few more exponential family properties that are useful for
gradient-based optimization algorithms and variational inference.
In particular, we note that the Fisher information matrix of an exponential
family can be computed as the Hessian matrix of its log partition function, and
that the KL divergence between two members of the same exponential family has a
simple expression.

\begin{definition}[Score vector and Fisher information matrix]
    Given a family of densities $p(x \given \theta)$ indexed by a parameter
    $\theta$, the \emph{score} vector $v(x, \theta)$ is the gradient of the log
    density with respect to the parameter,
    \begin{align}
        v(x, \theta) \triangleq \nabla_\theta \log p(x \given \theta),
    \end{align}
    and the \emph{Fisher information matrix} for the parameter $\theta$ is the
    covariance of the score,
    \begin{align}
        I(\theta) \triangleq \E \left[ v(x, \theta) v(x, \theta)^\T \right],
    \end{align}
    where the expectation is taken over the random variable $x$ with density
    $p(x \given \theta)$, and where we have used the identity $\E[v(x, \theta)]
    = 0$.
\end{definition}

\begin{proposition}[Score and Fisher information for exponential families]
    Given an exponential family of densities $p(x \given \sectwoeta)$ indexed by the
    natural parameter $\sectwoeta$, as in Eq.~\eqref{eq:natural_exp_fam}, the score
    with respect to the natural parameter is given by
    \begin{align}
        v(x, \sectwoeta) = \nabla_\sectwoeta \log p(x \given \sectwoeta) = t(x) - \nabla \log Z(\sectwoeta)
    \end{align}
    and the Fisher information matrix is given by
    \begin{align}
        I(\sectwoeta) = \nabla^2 \log Z(\sectwoeta).
    \end{align}
\end{proposition}

\begin{proposition}[KL divergence in an exponential family]
\label{prop:exp_fam_kl_divergence}
    Given an exponential family of densities $p(x \given \sectwoeta)$ indexed by the
    natural parameter $\sectwoeta$, as in Eq.~\eqref{eq:natural_exp_fam}, and two
    particular members with natural parameters $\sectwoeta_1$ and $\sectwoeta_2$,
    respectively, the KL divergence from one to the other is
    \begin{align}
        \KL(p(x \given \sectwoeta_1) \; \| \; p(x \given \sectwoeta_2) )
        &\triangleq \E_{p(x \given \sectwoeta_1)} \left[\log \frac{p(x \given \sectwoeta_1)}{p(x \given \sectwoeta_2)} \right]
        \label{eq:exp_fam_kl_divergence}
        \\
        &= \langle \sectwoeta_1 - \sectwoeta_2, \; \nabla \log Z(\sectwoeta_1) \rangle - (\log Z(\sectwoeta_1) - \log Z(\sectwoeta_2)).
        \notag
    \end{align}
\end{proposition}

\section{Natural gradient SVI for exponential families}
\label{sec:svi}

In this section we give a derivation of the natural gradient stochastic
variational inference (SVI) method of \citet{hoffman2013stochastic} using our
notation. We extend the algorithm in Section~\ref{sec:svae}.

\subsection{SVI objective}

Let $p(x, y \given \theta)$ be an exponential family and $p(\theta)$ be its
corresponding natural exponential family prior as in
Definitions~\ref{def:exp_fam} and~\ref{def:conj_prior}, writing
\begin{align}
    p(\theta) &= \exp \left\{ \langle \prioreta_\theta, \, t_\theta(\theta) \rangle - \log Z_\theta(\prioreta_\theta) \right\}
    \\
    p(x, y \given \theta) &= \exp \left\{ \langle \prioreta_{xy}(\theta), \, t_{xy}(x, y) \rangle - \log Z_{xy}(\prioreta_{xy}(\theta)) \right\}
    \\
    &= \exp \left\{ \langle t_\theta(\theta), \, (t_{xy}(x, y), 1) \rangle \right\}
    \label{eq:used_conj}
\end{align}
where we have used $t_\theta(\theta) = \left( \prioreta_{xy}(\theta), -\log
Z_{xy}(\prioreta_{xy}(\theta)) \right)$ in Eq.~\eqref{eq:used_conj}.

Given a fixed observation $y$, for any density $q(\theta, x) =
q(\theta)q(x)$ we have
\begin{align}
    \log p(y) &= \E_{q(\theta)q(x)} \! \left[ \log \frac{p(\theta) p(x, y \given \theta)}{q(\theta) q(x)} \right] + \KL(q(\theta)q(x) \;\|\; p(\theta, x \given y))
    \\
    &\geq \E_{q(\theta)q(x)} \! \left[ \log \frac{p(\theta) p(x, y \given \theta)}{q(\theta) q(x)} \right]
\end{align}
where we have used the fact that the KL divergence is always nonnegative.
Therefore to choose $q(\theta)q(x)$ to minimize the KL divergence to the posterior
$p(\theta, x \given y)$ we define the mean field variational inference
objective as
\begin{align}
\label{eq:mean_field_objective}
    \L\left[ \, q(\theta)q(x) \, \right] \triangleq \E_{q(\theta)q(x)} \! \left[ \log \frac{p(\theta) p(x, y \given \theta)}{q(\theta) q(x)} \right]
\end{align}
and the mean field variational inference problem as
\begin{align}
    {\max}_{q(\theta)q(x)} \L\left[ \, q(\theta)q(x) \, \right].
    \label{eq:mean_field_problem}
\end{align}

The following proposition shows that because of the exponential family
conjugacy structure, we can fix the parameterization of $q(\theta)$ and
still optimize over all possible densities without loss of generality.

\begin{proposition}[Optimal form of the global variational factor]
\label{prop:opt_global_factor}
    Given the mean field optimization problem Eq.~\eqref{eq:mean_field_problem}, for
    any fixed $q(x)$ the optimal factor $q(\theta)$ is detetermined
    ($\nu_\Theta$-a.e.) by
    \begin{align}
        q(\theta) \propto \exp \left\{ \langle \prioreta_\theta + \E_{q(x)} \left[ \, (t_{xy}(x, y), 1) \, \right], t_\theta(\theta) \rangle \right\}.
    \end{align}
    In particular, the optimal $q(\theta)$ is in the same exponential family
    as the prior $p(\theta)$.
\end{proposition}

This proposition follows immediately from a more general lemma, which we reuse
in the sequel.

\begin{lemma}[Optimizing a mean field factor]
\label{lem:optimizing_conjugate_mean_field_factor}
  Let $p(a, b, c)$ be a joint density and let $q(a)$, $q(b)$, and $q(c)$ be mean
  field factors.
  Consider the mean field variational inference objective
  \begin{align}
    \E_{q(a)q(b)q(c)} \! \left[ \log \frac{p(a,b,c)}{q(a)q(b)q(c)} \right].
  \end{align}
  For fixed $q(a)$ and $q(c)$, the partially optimal factor $\opt q(b)$ over
  all possible densities,
  \begin{align}
    \opt q(b) \triangleq \argmax_{q(b)}
    \E_{q(a)q(b)q(c)} \! \left[ \log \frac{p(a,b,c)}{q(a)q(b)q(c)} \right],
    \label{eq:lemma_objective}
  \end{align}
  is defined (almost everywhere) by
  \begin{align}
    \opt q(b) \propto \exp\left\{ \E_{q(a)q(c)} \log p(a, b, c) \right\}
    % = \exp \left\{ \E_{q(a)} \log p(b \given a) + \E_{q(a)q(c)} \log p(c \given b, a) \right\}
    % \notag
    .
  \end{align}
  In particular, if $p(c \given b, a)$ is an exponential family with $p(b
  \given a)$ its natural exponential family conjugate prior, and $\log p(b,c \given
  a)$ is a multilinear polynomial in the statistics $t_b(b)$ and $t_c(c)$, written
  \begin{align}
      p(b \given a)
      &=
      \exp \left\{ \langle \prioreta_b(a), \, t_b(b) \rangle - \log Z_b(\prioreta_b(a)) \right\},
      \\
      p(c \given b, a)
      &=
      \exp \left\{ \langle \prioreta_c(b, a), \, t_c(c) \rangle - \log Z_c(\prioreta_c(b, a)) \right\}
      \\
      &=
      \exp \left\{ \langle t_b(b), \prioreta_c(a)^\T (t_c(c), 1) \rangle \right\},
  \end{align}
  for some matrix $\prioreta_c(a)$, then the optimal factor can be written
  \begin{align}
    \opt q(b) &= \exp \left\{ \langle \vareta_b^*, \, t_b(b) \rangle - \log Z_b(\vareta_b^*) \right\},
    &
    \vareta_b^*
    &\triangleq
    \E_{q(a)} \prioreta_b(a) + \E_{q(a) q(c)} \prioreta_c(a)^\T (t_c(c), 1).
  \end{align}
  As a special case, when $c$ is conditionally independent of $b$ given $a$, so
  that $p(c \given b, a) = p(c \given b)$, then
  \begin{align}
    p(c \given b) &= \exp \left\{ \langle t_b(b), (t_c(c), 1) \rangle \right\},
    &
    \eta_b^*
    &\triangleq \E_{q(a)} \prioreta_b(a) + \E_{q(c)} (t_c(c), 1).
  \end{align}
\end{lemma}

\begin{proof}
  Rewrite the objective in Eq.~\eqref{eq:lemma_objective}, dropping terms that
  are constant with respect to $q(b)$, as
  \begin{align}
    \E_{q(a)q(b)q(c)} \! \left[ \log \frac{p(a, b, c)}{q(b)} \right]
    &=
    \E_{q(b)} \left[ \E_{q(a)q(c)} \log p(a, b, c) - \log q(b) \right]
    \\
    &=
    \E_{q(b)} \left[ \log \exp \E_{q(a)q(c)} \log p(a, b, c) - \log q(b) \right]
    \\
    &=
    -\E_{q(b)} \! \left[ \frac{q(b)}{\widetilde p(b)} \right] + \text{const}
    \\
    &= -\KL(q(b) \,\|\, \widetilde{p}(b)) + \text{const},
  \end{align}
  where we have defined a new density $\widetilde{p}(b) \propto \exp \left\{
  \E_{q(a)q(c)} \log p(a, b, c) \right\}$.
  We can maximize the objective by setting the KL divergence to zero, choosing
  $q(b) \propto \exp \left\{ \E_{q(a)q(c)} \log p(a, b, c) \right\}$.
  The rest follows from plugging in the exponential family densities.
\end{proof}

Proposition~\ref{prop:opt_global_factor} justifies parameterizing the density $q(\theta)$ with variational natural parameters $\vareta_\theta$ as
\begin{align}
    q(\theta) = \exp \left\{ \langle \vareta_\theta, \, t_\theta(\theta) \rangle - \log Z_\theta(\vareta_\theta) \right\}
\end{align}
where the statistic function $t_\theta$ and the log partition
function $\log Z_\theta$ are the same as in the prior family $p(\theta)$.
Using this parameterization, we can define the mean field objective as a
function of the parameters $\vareta_\theta$, partially optimizing over
$q(x)$,
\begin{align}
    \L(\vareta_\theta) \triangleq \max_{q(x)} \E_{q(\theta) q(x)} \! \left[ \log \frac{p(\theta) p(x, y \given \theta)}{q(\theta) q(x)} \right].
    \label{eq:svi_objective}
\end{align}
The partial optimization over $q(x)$ in Eq.~\eqref{eq:svi_objective} should be
read as choosing $q(x)$ to be a local partial optimizer of
Eq.~\eqref{eq:mean_field_objective}; in general, it may be intractable to find
a global partial optimizer, and the results that follow use only first-order stationary
conditions on $q(x)$.
We refer to this objective function, where we locally partially optimize the
mean field objective Eq.~\eqref{eq:mean_field_objective} over $q(x)$, as the SVI objective.

\subsection{Easy natural gradients of the SVI objective}

By again leveraging the conjugate exponential family structure, we can write a
simple expression for the gradient of the SVI objective, and even for its
natural gradient.

\begin{proposition}[Gradient of the SVI objective]
\label{prop:svi_gradient}
    Let the SVI objective $\L(\vareta_\theta)$ be defined as in
    Eq.~\eqref{eq:svi_objective}. Then the gradient $\nabla
    \L(\vareta_\theta)$ is
    \begin{align}
        \nabla \L(\vareta_\theta)
        =
        \left( \nabla^2 \log Z_\theta(\vareta_\theta) \right) \left( \prioreta_\theta + \E_{\opt q(x)} \left[ \, (t_{xy}(x, y), 1) \, \right] - \vareta_\theta \right)
    \end{align}
    where $\opt q(x)$ is a local partial optimizer of
    the mean field objective Eq.~\eqref{eq:mean_field_objective} for fixed
    global variational parameters $\vareta_\theta$.
\end{proposition}

\begin{proof}
    First, note that because $\opt q(x)$ is a local partial optimizer for
    Eq.~\eqref{eq:mean_field_objective} by
    Proposition~\ref{prop:gradients_of_locally_partial_opt}, we have
    \begin{align}
        \nabla \L(\vareta_\theta) = \nabla_{\vareta_\theta} \E_{q(\theta) \opt q(x)} \! \left[ \log \frac{p(\theta) p(x,y \given \theta)}{q(\theta) \opt q(x)} \right].
    \end{align}
    Next, we use the conjugate exponential family structure and
    Proposition~\ref{prop:conjugate_densities},
    Eq.~\eqref{eq:conjugate_joint_density}, to expand
    \begin{align}
        \E_{q(\theta)\opt q(x)} \! \left[ \log \frac{p(\theta) p(x,y \given \theta)}{q(\theta) \opt q(x)} \right]
        &=
        \langle \prioreta_\theta +\E_{\opt q(x)} (t_{xy}(x, y), 1) - \vareta_\theta, \; \E_{q(\theta)} [ t_\theta(\theta) ] \rangle
        \notag
        \\
        &\qquad - \left( \log Z_\theta(\prioreta_\theta) - \log Z_\theta(\vareta_\theta) \right).
    \end{align}
    Note that we can use Proposition~\ref{prop:gradient_of_log_Z} to replace
    $\E_{q(\theta)}[t_\theta(\theta)]$ with $\nabla \log
    Z_{\theta}(\vareta_\theta)$.
    Differentiating with respect to $\vareta_\theta$ and using the
    product rule, we have
    \begin{align}
        \nabla \L(\vareta_\theta)
        &=
        \nabla^2 \log Z_\theta(\vareta_\theta) \left( \prioreta_\theta +\E_{\opt q(x)}(t_{xy}(x, y), 1) - \vareta_\theta \right)
        \notag
        \\
        &\qquad - \nabla \log Z_\theta(\vareta_\theta) + \nabla \log Z_\theta(\vareta_\theta)
        \\
        &=
        \nabla^2 \log Z_\theta(\vareta_\theta) \left( \prioreta_\theta + \E_{\opt q(x)}(t_{xy}(x, y), 1) - \vareta_\theta \right).
    \end{align}
\end{proof}

As an immediate result of Proposition~\ref{prop:svi_gradient}, the natural
gradient \citep{amari1998natural} defined by
\begin{align}
    \widetilde{\nabla} \L(\vareta_\theta) \triangleq \left( \nabla^2 \log Z_\theta(\vareta_\theta) \right)^{-1} \nabla \L(\vareta_\theta)
\end{align}
has an even simpler expression.

\begin{corollary}[Natural gradient of the SVI objective]
\label{cor:svi_natural_gradient}
    The natural gradient of the SVI objective Eq.~\eqref{eq:svi_objective} is
    \begin{align}
        \widetilde{\nabla} \L(\vareta_\theta)
        =
        \prioreta_\theta + \E_{\opt q(x)} \! \left[ \, (t_{xy}(x, y), 1) \, \right]
        - \vareta_\theta.
        \label{eq:svi_natural_gradient}
    \end{align}
\end{corollary}

The natural gradient corrects for a kind of curvature in the variational family
and is invariant to reparameterization of the family \citep{amari2007methods}.
As a result, natural gradient ascent is effectively a second-order quasi-Newton
optimization algorithm, and using natural gradients can greatly accelerate the
convergence of gradient-based optimization algorithms \citep{martens2015kfac,
martens2015perspectives}.
It is a remarkable consequence of the exponential family structure that natural
gradients of the partially optimized mean field objective with respect to the
global variational parameters can be computed efficiently (without any backward
pass as would be required in generic reverse-mode differentiation).
Indeed, the exponential family conjugacy structure makes the natural gradient
of the SVI objective even easier to compute than the flat gradient.

\subsection{Stochastic natural gradients for large datasets}

The real utility of natural gradient SVI is in its application to large datasets.
Consider the model composed of global latent variables $\theta$, local latent variables $x = \{x_n\}_{n=1}^N$, and data $y = \{y_n\}_{n=1}^N$,
\begin{align}
    p(\theta, x, y) = p(\theta) \prod_{n=1}^N p(x_n, y_n \given \theta),
    \label{eq:svi_model_multiple_observations}
\end{align}
where each $p(x_n, y_n \given \theta)$ is a copy of the same likelihood
function with conjugate prior $p(\theta)$.
For fixed observations $y = \{y_n\}_{n=1}^N$,
let
\begin{align}
    q(\theta, x) = q(\theta) \prod_{n=1}^N q(x_n)
    \label{eq:svi_variational_family_multiple_observations}
\end{align}
be a variational family to
approximate the posterior $p(\theta, x \given y)$ and consider the SVI
objective given by Eq.~\eqref{eq:svi_objective}.
Using Eq.~\eqref{eq:conjugacy_multiple_likelihoods} of
Proposition~\ref{prop:conjugate_densities}, it is straightforward to extend the
natural gradient expression in Corollary~\ref{cor:svi_natural_gradient} to an
unbiased Monte Carlo estimate which samples terms in the sum over data points.

\begin{corollary}[Unbiased Monte Carlo estimate of the SVI natural gradient]
\label{cor:svi_monte_carlo_natural_gradient}
    Using the model and variational family
    \begin{align}
        p(\theta, x, y) &= p(\theta) \prod_{n=1}^N p(x_n, y_n \given \theta),
        &
        q(\theta) q(x) &= q(\theta) \prod_{n=1}^N q(x_n),
    \end{align}
    where $p(\theta)$ and $p(x_n, y_n \given \theta)$ are a conjugate pair of
    exponential families,
    define $\L(\vareta_\theta)$ as in Eq.~\eqref{eq:svi_objective}.
    Let the random index $\hat n$ be sampled from the set $\{1,2, \ldots,
    N\}$ and let $p_n > 0$ be the probability it takes value $n$.
    Then
    \begin{align}
        \widetilde\nabla \L(\vareta_\theta)
        &=
        \E_{\hat n} \left[
        \prioreta_\theta +
        \frac{1}{p_{\hat n}} \E_{\opt q(x_{\hat n})} \!\! \left[ \, (t_{xy}(x_{\hat n}, y_{\hat n}), 1) \, \right]
        - \vareta_\theta
        \right]
        ,
    \end{align}
    where $\opt q(x_{\hat n})$ is a local partial optimizer of $\L$ given
    $q(\theta)$.
\end{corollary}

\begin{proof}
    Taking expectation over the index $\hat n$, we have
    \begin{align}
        \E_{\hat n} \! \! \left[ \frac{1}{p_{\hat n}} \E_{\opt q(x_{\hat n})} \! \left[ \, (t_{xy}(x_{\hat n}, y_{\hat n}), 1) \, \right] \right]
        &=
        \sum_{n=1}^N \frac{p_n}{p_n} \E_{\opt q(x_{n})} \! \left[ \, (t_{xy}(x_{n}, y_{n}), 1) \, \right]
        \\
        &=
        \sum_{n=1}^N \E_{\opt q(x_n)} \! \left[ \, (t_{xy}(x_n, y_n), 1) \, \right].
    \end{align}
    The remainder of the proof follows from
    Proposition~\ref{prop:conjugate_densities} and the same argument as in
    Proposition~\ref{prop:svi_gradient}.
\end{proof}

The unbiased stochastic gradient developed in
Corollary~\ref{cor:svi_monte_carlo_natural_gradient} can be used in a scalable
stochastic gradient ascent algorithm.
To simplify notation, in the following sections we drop the notation for
multiple likelihood terms $p(x_n, y_n \given \theta)$ for $n=1,2,\ldots,N$ and
return to working with a single likelihood term $p(x, y \given \theta)$.
The extension to multiple likelihood terms is immediate.

\subsection{Conditinally conjugate models and block updating}
\label{sec:conditionally_conjugate_local_optimization}

The model classes often considered for natural gradient SVI, and the main model
classes we consider here, have additional conjugacy structure in the local
latent variables.
In this section we introduce notation for this extra structure in terms of the
additional local latent variables $z$ and discuss the local block coordinate
optimization that is often performed to compute the factor $\opt q(z) \opt q(x)$ for
use in the natural gradient expression.

Let $p(z, x, y \given \theta)$ be an exponential family and $p(\theta)$ be its
corresponding natural exponential family conjugate prior, writing
\begin{align}
    \label{eq:svi_conditionally_conj_start}
    p(\theta) &= \exp \left\{ \langle \prioreta_\theta, \, t_\theta(\theta) \rangle - \log Z_\theta(\prioreta_\theta) \right\},
    \\
    p(z, x, y \given \theta) &= \exp \left\{ \langle \prioreta_{zxy}(\theta), \, t_{zxy}(z, x, y) \rangle - \log Z_{zxy}(\prioreta_{zxy}(\theta)) \right\}
    \\
    &= \exp \left\{\langle t_\theta(\theta), \, (t_{zxy}(z, x, y), 1) \rangle \right\},
    \label{eq:used_conj2}
\end{align}
where we have used $t_\theta(\theta) = \left( \prioreta_{zxy}(\theta), -\log Z_{zxy}(\prioreta_{zxy}(\theta)) \right)$ in Eq.~\eqref{eq:used_conj2}.
Additionally,
let $t_{zxy}(z, x, y)$ be a multilinear polynomial in the statistics functions $t_x(x)$, $t_y(y)$, and $t_z(z)$,
let $p(z \given \theta)$, $p(x \given z, \theta)$, and $p(y
\given x, z, \theta) = p(y \given x, \theta)$ be exponential families, and let $p(z \given \theta)$ be
a conjugate prior to $p(x \given z, \theta)$ and $p(x \given z, \theta)$ be a
conjugate prior to $p(y \given x, \theta)$, so that
\begin{align}
    p(z \given \theta) &= \exp \left\{ \langle \prioreta_z(\theta), \, t_z(z) \rangle - \log Z_z(\prioreta_z(\theta)) \right\},
    \label{eq:svi_conditionally_conj_mid_1} \noeqref{eq:svi_conditionally_conj_mid_1}
    \\
    p(x \given z, \theta) &= \exp \left\{ \langle \prioreta_x(z, \theta), \, t_x(x) \rangle - \log Z_x(\prioreta_x(z, \theta)) \right\}
    \\
    &= \exp \left\{ \langle t_z(z), \, \prioreta_x(\theta)^\T (t_x(x), 1) \rangle \right\},
    \label{eq:svi_conditionally_conj_mid} \noeqref{eq:svi_conditionally_conj_mid}
    \\
    p(y \given x, \theta) &= \exp \left\{ \langle \prioreta_y(x, \theta), \, t_y(y) \rangle - \log Z_y(\prioreta_y(x, z, \theta)) \right\}
    \\
    &= \exp \left\{ \langle t_x(x), \, \prioreta_y(\theta)^\T (t_y(y), 1) \rangle \right\},
    \label{eq:svi_conditionally_conj_end}
\end{align}
for some matrices $\prioreta_x(\theta)$ and $\prioreta_y(\theta)$.

This model class includes many common models, including the latent Dirichlet
allocation, switching linear dynamical systems with linear-Gaussian emissions,
and mixture models and hidden Markov models with exponential family emissions.
The conditionally conjugate structure is both powerful and restrictive: while
it potentially limits the expressiveness of the model class, it enables
block coordinate optimization with very simple and fast updates, as we
show next.
When conditionally conjugate structure is not present, these local
optimizations can instead be performed with generic gradient-based methods and
automatic differentiation \citep{DuvenaudAdams2015bbsvi}.

\begin{proposition}[Unconstrained block coordinate ascent on $q(z)$ and $q(x)$]
\label{prop:local_block_coordinate_ascent}
    Let $p(\theta, z, x, y)$ be a model as in
    Eqs.~\eqref{eq:svi_conditionally_conj_start}-\eqref{eq:svi_conditionally_conj_end},
    and for fixed data $y$ let $q(\theta)q(z)q(x)$ be a corresponding mean
    field variational family for approximating the posterior $p(\theta, z, x
    \given y)$, with
    \begin{align}
        q(\theta) &= \exp \left\{ \langle \vareta_\theta, \, t_\theta(\theta) \rangle - \log Z_\theta(\vareta_\theta) \right\},
        \\
        q(z) &= \exp \left\{ \langle \vareta_z, \, t_z(z) \rangle - \log Z_z(\vareta_z) \right\},
        \\
        q(x) &= \exp \left\{ \langle \vareta_x, \, t_x(x) \rangle - \log Z_x(\vareta_x) \right\},
    \end{align}
    and with
    the mean field variational inference objective
    \begin{align}
        \L [ \, q(\theta) q(z) q(x) \, ]
        =
        \E_{q(\theta)q(z)q(x)} \! \left[ \log \frac{p(\theta)p(z \given \theta)p(x \given z, \theta) p(y \given x, z, \theta)}{q(\theta)q(z)q(x)} \right].
    \end{align}
    Fixing the other factors, the partial optimizers $\opt q(z)$ and $\opt
    q(x)$ for $\L$ over all possible densities are given by
    \begin{align}
        \opt q(z)
        &\triangleq
        \argmax_{q(z)} \L[ \, q(\theta) q(z) q(x) \, ]
        =
        \exp \left\{ \langle \vareta^*_z, \, t_z(z) \rangle - \log Z_z(\vareta^*_z) \right\},
        \\
        \opt q(x)
        &\triangleq
        \argmax_{q(x)} \L[ \, q(\theta) q(z) q(x) \, ]
        =
        \exp \left\{ \langle \vareta^*_x, \, t_x(x) \rangle - \log Z_x(\vareta^*_x) \right\},
    \end{align}
    with
    \begin{align}
        \vareta^*_z
        &=
        \E_{q(\theta)} \prioreta_z(\theta) + \E_{q(\theta)q(x)} \prioreta_x(\theta)^\T (t_x(x), 1),
        \label{eq:svi_block_coord_update_1}
        \\
        \vareta^*_x
        &=
        \E_{q(\theta)q(z)} \prioreta_x(\theta) t_z(z) + \E_{q(\theta)} \prioreta_y(\theta)^\T (t_y(y), 1).
        \label{eq:svi_block_coord_update}
    \end{align}
    % TODO kronecker factor structure here
\end{proposition}

\begin{proof}
  This proposition is a consequence of
  Lemma~\ref{lem:optimizing_conjugate_mean_field_factor} and the conjugacy
  structure.
\end{proof}

Proposition~\ref{prop:local_block_coordinate_ascent} gives an
efficient block coordinate ascent algorithm:
for fixed $\vareta_\theta$, by alternatively updating $\vareta_z$
and $\vareta_x$ according to Eqs.~\eqref{eq:svi_block_coord_update_1}-\eqref{eq:svi_block_coord_update} we
are guaranteed to converge to a stationary point that is partially optimal in
the parameters of each factor.
In addition, performing each update requires only computing expected sufficient
statistics in the variational factors, which means evaluating $\nabla \log
Z_\theta(\vareta_\theta)$, $\nabla \log Z_z(\vareta_z)$, and
$\nabla \log Z_x(\vareta_x)$, quantities that be computed
anyway in a gradient-based optimization routine.
The block coordinate ascent procedure leveraging this conditional conjugacy
structure is thus not only efficient but also does not require a choice of step
size.

Note in particular that this procedure produces parameters
$\vareta_z^*(\vareta_\theta)$ and $\vareta_x^*(\vareta_\theta)$ that are
partially optimal (and hence stationary) for the objective.
That is, defining the parameterized mean field variational inference
objective as $L(\vareta_\theta, \vareta_z, \vareta_x) =
\L[\, q(\theta)q(z)q(x) \, ]$, for fixed $\vareta_\theta$ the block
coordinate ascent procedure has limit points $\vareta_z^*$ and
$\vareta_x^*$ that satisfy
\begin{align}
  \nabla_{\vareta_z} \L(\vareta_\theta, \vareta_z^*(\vareta_\theta),
  \vareta_x^*(\vareta_\theta))
    &= 0,
    &
  \nabla_{\vareta_x} \L(\vareta_\theta, \vareta_z^*(\vareta_\theta),
  \vareta_x^*(\vareta_\theta))
    &= 0.
\end{align}

\section{The SVAE objective and its gradients}
\label{sec:svae}

In this section we define the SVAE variational lower bound and show how to
efficiently compute unbiased stochastic estimates of its gradients, including
an unbiased estimate of the natural gradient with respect to the
variational parameters with conjugacy structure.
The setup here parallels the setup for natural gradient SVI in
Section~\ref{sec:svi}, but while SVI is restricted to complete-data conjugate
models, here we consider more general likelihood models.

% \subsection{The SVAE objective \texorpdfstring{$\Lsvae(\vareta_\theta, \vareta_\gamma, \phi)$}{}}

\subsection{SVAE objective}
\label{appendix:svae_objective}

Let $p(x \given \theta)$ be an exponential family and let $p(\theta)$ be its
corresponding natural exponential family conjugate prior, as in
Definitions~\ref{def:exp_fam} and~\ref{def:conj_prior}, writing
\begin{align}
    p(\theta) &= \exp \left\{ \langle \prioreta_\theta, t_\theta(\theta) \rangle - \log Z_\theta(\prioreta_\theta) \right\},
    \label{eq:appendix:svae_densities_start}
    \\
    p(x \given \theta) &= \exp \left\{ \langle \prioreta_{x}(\theta), t_{x}(x) \rangle - \log Z_{x}(\prioreta_{x}(\theta)) \right\}
    \\
    &= \exp \left\{ \langle t_\theta(\theta), (t_{x}(x), 1) \rangle \right\},
    \label{eq:appendix:svae_densities_end}
\end{align}
where we have used $t_\theta(\theta) = \left( \prioreta_{x}(\theta), -\log
Z_{x}(\prioreta_{x}(\theta)) \right)$ in Eq.~\eqref{eq:appendix:svae_densities_end}.
Let $p(y \given x, \gamma)$ be a general family of densities (not necessarily an
exponential family) and let $p(\gamma)$ be an exponential family prior on
its parameters of the form
\begin{align}
    p(\gamma) = \exp \left\{ \langle \prioreta_\gamma, \,
    t_\gamma(\gamma) \rangle - \log Z_\gamma(\prioreta_\gamma) \right\}.
\end{align}

For fixed $y$, consider the mean field family of densities $q(\theta,
\gamma, x) = q(\theta)q(\gamma)q(x)$ and the mean field variational
inference objective
\begin{align}
    \L[ \, q(\theta) q(\gamma) q(x) \, ]
    \triangleq \E_{q(\theta)q(\gamma)q(x)} \! \left[ \log \frac{p(\theta)p(\gamma)p(x \given \theta) p(y \given x, \gamma)}{q(\theta)q(\gamma)q(x)} \right].
    \label{eq:appendix:svae_mean_field_objective}
\end{align}
By the same argument as in Proposition~\ref{prop:opt_global_factor}, without
loss of generality we can take the global factor $q(\theta)$ to be in the same
exponential family as the prior $p(\theta)$, and we denote its natural
parameters by $\vareta_\theta$, writing
\begin{align}
    q(\theta) = \exp \left\{ \langle \vareta_\theta, t_\theta(\theta) \rangle - \log Z_\theta(\vareta_\theta) \right\}.
\end{align}
We restrict $q(\gamma)$ to be in the same exponential family as
$p(\gamma)$ with natural parameters $\vareta_\gamma$, writing
\begin{align}
    q(\gamma) = \exp \left\{ \langle \vareta_\gamma, t_\gamma(\gamma) \rangle - \log Z_\gamma(\vareta_\gamma) \right\}.
\end{align}
Finally, we restrict%
\footnote{The parametric form for $q(x)$ need not be restricted a priori, but
rather without loss of generality given the surrogate objective
Eq.~\eqref{eq:svae_surrogate_objective} and the form of $\psi$ used in
Eq.~\eqref{eq:psi}, the optimal factor $q(x)$ is in the same family as $p(x
\given \theta)$. We treat it as a restriction here so that we can proceed with
more concrete notation.}
$q(x)$ to be in the same exponential family as $p(x \given
\theta)$, writing its natural parameter as $\vareta_x$.
Using these explicit variational natural parameters, we rewrite the mean field variational inference objective in Eq.~\eqref{eq:appendix:svae_mean_field_objective} as
\begin{align}
    \L(\vareta_\theta, \vareta_\gamma, \vareta_x)
    \triangleq \E_{q(\theta)q(\gamma)q(x)} \! \left[ \log \frac{p(\theta)p(\gamma)p(x \given \theta) p(y \given x, \gamma)}{q(\theta)q(\gamma)q(x)} \right].
    \label{eq:appendix:svae_mean_field_objective_parameterized}
\end{align}

To perform efficient optimization in the objective
$\L$ defined in Eq.~\eqref{eq:appendix:svae_mean_field_objective_parameterized}, we consider choosing the
variational parameter $\vareta_x$ as a function of the other
parameters $\vareta_\theta$ and $\vareta_\gamma$.
One natural choice is to set $\vareta_x$ to be a local partial
optimizer of $\L$, as in Section~\ref{sec:svi}.
However, finding a local partial optimizer may be computationally expensive
for general densities $p(y \given x, \gamma)$, and in the large data setting
this expensive optimization would have to be performed for each stochastic
gradient update.
Instead, we choose $\vareta_x$ by optimizing over a
surrogate objective $\widehat \L$, which we design using exponential family
structure to be both easy to optimize and to share curvature properties with
the mean field objective $\L$.
The surrogate objective $\widehat \L$ is
\begin{align}
    \widehat \L(\vareta_\theta, \vareta_\gamma, \vareta_x, \phi)
    &\triangleq \E_{q(\theta)q(\gamma)q(x)} \! \left[ \log \frac{p(\theta)p(\gamma)p(x \given \theta) \exp \{ \psi(x ; y, \phi) \} }{q(\theta) q(\gamma) q(x)} \right]
    \\
    &= \E_{q(\theta)q(x)} \! \left[ \log \frac{p(\theta)p(x \given \theta) \exp \{ \psi(x ; y, \phi) \} }{q(\theta) q(x)} \right] + \mathrm{const},
    \label{eq:svae_surrogate_objective}
\end{align}
where the constant does not depend on $\vareta_x$.
We define the function $\psi(x ; y, \phi)$ to have a form related to the
exponential family $p(x \given \theta)$,
\begin{align}
    \psi(x ; y, \phi) \triangleq \langle r(y ; \phi), \; t_x(x) \rangle,
    \label{eq:psi}
\end{align}
where $\{ r(y ; \phi) \}_{\phi \in \R^m}$ is some class of functions
parameterized by $\phi \in \R^m$, which we assume
only to be continuously differentiable in $\phi$.
We call $r(y; \phi)$ the \emph{recognition model}.
We define $\vareta^*_x(\vareta_\theta, \phi)$ to be a local
partial optimizer of $\widehat \L$,
\begin{align}
    \vareta^*_x(\vareta_\theta, \phi) \triangleq \argmin_{\vareta_x} \widehat \L (\vareta_\theta, \vareta_\gamma, \vareta_x, \phi),
\end{align}
where the notation above should be interpreted as choosing
$\vareta^*_x(\vareta_\theta, \phi)$ to be a local argument of
maximum.
The results to follow rely only on
% the fact that
% $\vareta^*_x(\vareta_\theta, \phi)$ satisfies the local
% stationary conditions $\nabla_{\vareta_x} \widehat L (\vareta_\theta,
% \vareta_\gamma, \vareta^*_x(\vareta_\theta,
% \phi), \phi) = 0$,
% which are
necessary first-order conditions for unconstrained
local optimality.

Given this choice of function $\vareta^*_x(\vareta_\theta, \phi)$, we define the SVAE objective to be
\begin{align}
    \Lsvae(\vareta_\theta, \vareta_\gamma, \phi) \triangleq \L(\vareta_\theta, \vareta_\gamma, \vareta^*_x(\vareta_\theta, \phi)),
    \label{eq:svae_objective}
\end{align}
where $\L$ is the mean field variational inference defined in
Eq.~\eqref{eq:appendix:svae_mean_field_objective_parameterized}, and we define the SVAE
optimization problem to be
\begin{align}
    {\max}_{\vareta_\theta, \vareta_\gamma, \phi} \Lsvae(\vareta_\theta, \vareta_\gamma, \phi).
\end{align}

We summarize these definitions in the following.

\begin{definition}[SVAE objective]
    Let $\L$ denote the mean field variational inference objective
    \begin{align}
        \L[ \, q(\theta) q(\gamma) q(x) \, ]
        \triangleq
        \E_{q(\theta)q(\gamma)q(x)} \! \left[ \log \frac{p(\theta)p(\gamma)p(x \given \theta) p(y \given x, \gamma)}{q(\theta)q(\gamma)q(x)} \right],
        \label{eq:appendix:svae_mean_field_objective_2}
    \end{align}
    where the densities $p(\theta)$, $p(\gamma)$, and $p(x \given \theta)$
    are exponential families and $p(\theta)$ is the natural exponential family
    conjugate prior to $p(x \given \theta)$,
    as in Eqs.~\eqref{eq:appendix:svae_densities_start}-\eqref{eq:appendix:svae_densities_end}.
    Given a parameterization of the variational factors as
    \begin{gather}
        q(\theta)
        =
        \exp \left\{ \langle \vareta_\theta, \, t_\theta(\theta) \rangle - \log Z_\theta(\vareta_\theta) \right\},
        \quad
        q(\gamma)
        =
        \exp \left\{ \langle \vareta_\gamma, \, t_\gamma(\gamma)
        \rangle - \log Z_\gamma(\vareta_\gamma) \right\},
        \notag
        \\
        q(x)
        =
        \exp \left\{ \langle \vareta_x, \, t_x(x) \rangle - \log Z_x(\vareta_x) \right\},
    \end{gather}
    let
    $\L(\vareta_\theta, \vareta_\gamma, \vareta_x)$
    denote the mean field variational inference objective Eq.~\eqref{eq:appendix:svae_mean_field_objective_2} as a function of these
    variational parameters.
    We define the \emph{SVAE objective} as
    \begin{align}
        \Lsvae(\vareta_\theta, \vareta_\gamma, \phi)
        &\triangleq
        \L(\vareta_\theta, \vareta_\gamma, \vareta_x^*(\vareta_\theta, \phi)),
    \end{align}
    where $\vareta_x^*(\vareta_\theta, \phi)$ is defined
    as a local partial optimizer of the surrogate objective $\widehat\L$,
    \begin{align}
      \vareta_x^*(\vareta_\theta, \phi) \triangleq \argmax_{\vareta_x}
      \widehat \L(\vareta_\theta, \vareta_x^*(\vareta_\theta, \phi), \phi),
    \end{align}
    where the surrogate objective $\widehat \L$ is defined as
    \begin{align}
        \widehat \L(\vareta_\theta, \vareta_x, \phi)
        &\triangleq
        \E_{q(\theta)q(x)} \! \left[ \log \frac{p(\theta)p(x \given \theta) \exp \{ \psi(x ; y, \phi) \} }{q(\theta) q(x)} \right],
        \\
        \psi(x; y, \phi)
        &\triangleq
        \langle r(y; \phi), \, t_x(x) \rangle,
    \end{align}
    for some \emph{recognition model} $r(y ; \phi)$ parameterized by $\phi \in \R^m$.
\end{definition}

The SVAE objective $\Lsvae$ is a lower-bound for the partially-optimized mean
field variational inference objective in the following sense.

\begin{proposition}[The SVAE objective lower-bounds the mean field objective]
\label{prop:appendix:svae_lower_bound}
    The SVAE objective function $\Lsvae$ lower-bounds the partially-optimized
    mean field objective $\L$ in the sense that
    \begin{align}
        \max_{q(x)} \L [ \, q(\theta) q(\gamma) q(x) \, ]
        \geq
        \max_{\vareta_x} \L(\vareta_\theta, \vareta_\gamma, \vareta_x)
        \geq 
        \Lsvae(\vareta_\theta, \vareta_\gamma, \phi)
        \quad
        \forall \phi \in \R^m,
    \end{align}
    for any choice of function class $\{r(y; \phi)\}_{\phi \in \R^m}$ in
    Eq.~\eqref{eq:psi}.
    Furthermore, if there is some $\phi^* \in \R^m$ such that
    \begin{align}
        \psi(x ; y, \phi^*) = \E_{q(\gamma)} \log p(y \given x, \gamma)
    \end{align}
    then the bound can be made tight in the sense that
    \begin{align}
        \max_{q(x)} \L [ \, q(\theta) q(\gamma) q(x) \, ]
        =
        \max_{\vareta_x} \L(\vareta_\theta, \vareta_\gamma, \vareta_x)
        =
        \max_\phi \Lsvae(\vareta_\theta, \vareta_\gamma, \phi).
        % =
        % \max_\phi
        % \L(\vareta_\theta, \vareta_\gamma,
        % \vareta^*_x(\vareta_\theta, \phi)).
    \end{align}
\end{proposition}

\begin{proof}
  The inequalities follow from the variational principle and the definition of
  the SVAE objective $\Lsvae$.
  In particular, by Lemma~\ref{lem:optimizing_conjugate_mean_field_factor}
  the optimal factor over all possible densities is given by
  \begin{align}
    q^{**}\!(x) \propto \exp \left\{ \langle \E_{q(\theta)} \prioreta_x(\theta), \, t_x(x) \rangle + \E_{q(\gamma)} \log p(y \given x, \gamma) \right\},
    \label{eq:svae_optimal_factor}
  \end{align}
  while we restrict the factor $q(x)$ to have a particular exponential family
  form indexed by parameter $\vareta_x$, namely
  $q(x) \propto \exp \left\{ \langle \vareta_x, \, t_x(x) \rangle \right\}$.
  In the definition of $\Lsvae$ we also restrict the parameter
  $\vareta_x$ to be set to $\vareta_x^*(\vareta_\theta,
  \phi)$, a particular function of $\vareta_\theta$ and $\phi$, rather
  than setting it to the value that maximizes the mean field objective $\L$.
  Finally, equality holds when we can set $\phi$ to match the optimal
  $\vareta_x$ and that choice yields the optimal factor given in
  Eq.~\eqref{eq:svae_optimal_factor}.
\end{proof}

Proposition~\ref{prop:appendix:svae_lower_bound} motivates the SVAE optimization
problem: by using gradient-based optimization to maximize
$\Lsvae(\vareta_\theta, \vareta_\gamma, \phi)$ we are
maximizing a lower-bound on the model evidence $\log p(y)$ and
correspondingly minimizing the KL divergence from our variational family to the
target posterior.
Furthermore, it motivates choosing the recognition model function class $\{ r(y;
\phi) \}_{\phi \in \R^m}$ to be as rich as possible.

As we show in the following, choosing
$\vareta^*_x(\vareta_\theta, \phi)$ to be a local partial
optimizer of the surrogate objective $\widehat \L$ provides two significant
computational advantages.
First, it allows us to provide a simple expression for an unbiased estimate of
the natural gradient $\widetilde\nabla_{\vareta_\theta} \Lsvae$, as we
describe next in Section~\ref{sec:svae_natural_gradient}.
Second, it allows $\vareta^*_x(\vareta_\theta, \phi)$ to be
computed efficiently by exploiting exponential family structure, as we show in
Section~\ref{sec:svae_local_factor_optimization}.

\subsection{Estimating the natural gradient \texorpdfstring{$\widetilde{\nabla}_{\vareta_\theta} \Lsvae$}{with respect to the conjugate global variational parameters}}
\label{sec:svae_natural_gradient}

The definition of $\vareta^*_x$ in terms of the surrogate objective
$\widehat \L$ enables computationally efficient ways to estimate
natural gradient with respect to the conjugate global variational parameters,
$\widetilde{\nabla}_{\vareta_\theta} \Lsvae(\vareta_\theta,
\vareta_\gamma, \phi)$.
The next proposition covers the case when the local latent variational factor
$q(x)$ has no additional factorization structure.

\begin{proposition}[Natural gradient of the SVAE objective]
\label{prop:appendix:svae_natural_gradient}
    When there is only one local latent variational factor $q(x)$ (and no
    further factorization structure),
    the natural gradient of
    the SVAE objective Eq.~\eqref{eq:svae_objective} with respect to the
    conjugate global variational parameters $\vareta_\theta$ is
    \begin{align}
        \widetilde\nabla_{\vareta_\theta} \Lsvae(\vareta_\theta, \vareta_\gamma, \phi)
        &=
        \left(\prioreta_\theta + \E_{\opt q(x)} \left[ (t_x(x), 1) \right] - \vareta_\theta \right)
        +
        (\nabla_{\vareta_x} \L(\vareta_\theta, \vareta_\gamma, \vareta^*_x(\vareta_\theta, \phi)), 0)
        \notag
    \end{align}
    where the first term is the SVI natural gradient from
    Corollary~\ref{cor:svi_natural_gradient}, using
    \begin{align}
        \opt q(x) \triangleq \exp \left\{ \langle \vareta^*_x(\vareta_\theta, \phi), \; t_x(x) \rangle - \log Z_x(\vareta^*_x(\vareta_\theta, \phi)) \right\},
    \end{align}
    and where a stochastic estimate of the second term
    is computed as part of the backward pass for the gradient
    $\nabla_\phi \L(\vareta_\theta, \vareta_\gamma,
    \vareta^*_x(\vareta_\theta, \phi))$.
\end{proposition}

\begin{proof}
  First we use the chain rule, analogously to Eq.~\eqref{eq:g_grad}, to write
  the gradient as
  \begin{align}
    \nabla_{\vareta_\theta} \Lsvae(\vareta_\theta, \vareta_\gamma, \phi)
    &=
    \left( \nabla^2 \log Z_\theta(\vareta_\theta) \right) \left( \prioreta_\theta + \E_{\opt q(x)} \left[ \, (t_{xy}(x, y), 1) \, \right] - \vareta_\theta \right)
    \notag
    \\
    &\quad +
    \left(\nabla_{\vareta_\theta} \vareta^*_x(\vareta_\theta, \phi) \right)
    \left( \nabla_{\vareta_x}\L(\vareta_\theta, \vareta_\gamma, \vareta_x^*(\vareta_\theta, \phi)) \right),
    \label{eq:svae_grad_1}
  \end{align}
  where the first term is the same as the SVI gradient derived in
  Proposition~\ref{prop:svi_gradient}.
  In the case of SVI, the second term is zero because $\vareta^*_x$ is
  chosen as a partial optimizer of $\L$, but for the SVAE objective the second
  term is nonzero in general, and the remainder of this proof amounts to
  deriving a simple expression for it.

  We compute the term $\nabla_{\vareta_\theta}
  \vareta^*_x(\vareta_\theta, \phi)$ in
  Eq.~\eqref{eq:svae_grad_1} in terms of the
  gradients of the surrogate objective $\widehat\L$ using the Implicit Function
  Theorem given in
  Corollary~\ref{cor:implicit_function_theorem_for_optimization}, which yields
  \begin{align}
    \nabla_{\vareta_\theta} \vareta^*_x(\vareta_\theta, \phi)
    =
    - \nabla^2_{\vareta_\theta \vareta_x} \widehat\L(\vareta_\theta, \vareta^*_x(\vareta_\theta, \phi), \phi)
    \left( \nabla^2_{\vareta_x \vareta_x} \widehat\L(\vareta_\theta, \vareta^*_x(\vareta_\theta, \phi), \phi) \right)^{-1}.
    \label{eq:svae_implicit_function_theorem}
  \end{align}
  First, we compute the gradient of $\widehat\L$ with respect to
  $\vareta_x$, writing
  \begin{align}
    \nabla_{\vareta_x} \widehat\L(\vareta_\theta, \vareta_x, \phi)
    &=
    \nabla_{\vareta_x}
    \left[
    \E_{q(\theta)q(x)} \! \left[ \log \frac{p(x \given \theta) \exp \{ \psi(x ; y, \phi) \})}{q(x)} \right] \right]
    \\
    &=
    \nabla_{\vareta_x}
    \left[
      \langle \E_{q(\theta)} \prioreta_x(\theta) + r(y ; \phi) - \vareta_x, \, \nabla \log Z_x(\vareta_x) \rangle + \log Z_x(\vareta_x)
    \right]
    \notag
    \\
    &=
    \left( \nabla^2 \log Z_x(\vareta_x) \right)
    \left( \E_{q(\theta)} \prioreta_x(\theta) + r(y ; \phi) - \vareta_x \right),
    \label{eq:svae_proof_grad_1}
  \end{align}

  When there is only one local latent variational factor $q(x)$ (and no further
  factorization structure),  as a consequence of the first-order stationary
  condition $\nabla_{\vareta_x} \widehat \L(\vareta_\theta,
  \vareta_x^*(\vareta_\theta, \phi), \phi) = 0$ and the fact that $\nabla^2
  \log Z_x(\vareta_x)$ is always positive definite for minimal exponential
  families, we have
  \begin{align}
    \E_{q(\theta)} \prioreta_x(\theta) + r(y ; \phi) -
    \vareta^*_x(\vareta_\theta, \phi) = 0,
    \label{eq:svae_proof_stationary_condition}
  \end{align}
  which is useful in simplifying the expressions to follow.

  Continuing with the calculation of the terms in
  Eq.~\eqref{eq:svae_implicit_function_theorem}, we compute
  $\nabla^2_{\vareta_x \vareta_x} \widehat\L$ by differentiating
  the expression in Eq.~\eqref{eq:svae_proof_grad_1} again, writing
  \begin{align}
    \nabla^2_{\vareta_x \vareta_x} \widehat\L(\vareta_\theta, \vareta_x^*(\vareta_\theta, \phi), \phi)
    &=
    - \nabla^2 \log Z_x(\vareta^*_x(\vareta_\theta, \phi))
    \label{eq:svae_natural_gradient_used_stationarity}
    \\
    &\;\;
    + \left( \nabla^3 \log Z_x(\vareta^*_x(\vareta_\theta, \phi)) \right)
    \!
    \left(
    \E_{q(\theta)} \prioreta_x(\theta) + r(y ; \phi) -
    \vareta^*_x(\vareta_\theta, \phi)
    \right)
    \notag
    \\
    &=
    - \nabla^2 \log Z_x(\vareta^*_x(\vareta_\theta, \phi)),
  \end{align}
  where the last line follows from using the first-order stationary condition
  Eq.~\eqref{eq:svae_proof_stationary_condition}.
  Next, we compute the other term $\nabla^2_{\vareta_\theta
  \vareta_x} \widehat\L$ by differentiating
  Eq.~\eqref{eq:svae_proof_grad_1} with respect to $\vareta_\theta$ to
  yield
  \begin{align}
    \nabla^2_{\vareta_\theta \vareta_x} \widehat\L(\vareta_\theta, \vareta_x^*(\vareta_\theta, \phi), \phi)
    &=
    \left( \nabla^2 \log Z_\theta(\vareta_\theta) \right)
    \begin{pmatrix}
    \nabla^2 \log Z_x(\vareta^*_x(\vareta_\theta, \phi))
    \\
    0
    \end{pmatrix},
  \end{align}
  where the latter matrix is $\nabla^2 \log
  Z_x(\vareta_x^*(\vareta_\theta, \phi))$ padded by a row of
  zeros.

  Plugging these expressions back into
  Eq.~\eqref{eq:svae_implicit_function_theorem} and cancelling, we arrive at
  \begin{align}
      \nabla_{\vareta_\theta} \vareta^*_x(\vareta_\theta, \phi)
      =
      \nabla^2 \log Z_\theta(\vareta_\theta)
      \begin{pmatrix}
        I
        \\
        0
      \end{pmatrix},
  \end{align}
  and so we have an expression for the gradient of the SVAE objective as
  \begin{align}
    \nabla_{\vareta_\theta} \Lsvae(\vareta_\theta, \vareta_\gamma, \phi)
    &=
    \left( \nabla^2 \log Z_\theta(\vareta_\theta) \right) \left( \prioreta_\theta + \E_{\opt q(x)} \left[ \, (t_{xy}(x, y), 1) \, \right] - \vareta_\theta \right)
    \notag
    \\
    &\quad +
    \left( \nabla^2 \log Z_\theta(\vareta_\theta) \right)
    \left( \nabla_{\vareta_x} \L(\vareta_\theta, \vareta_\gamma, \vareta_x^*(\vareta_\theta, \phi)), 0 \right).
  \end{align}
  When we compute the natural gradient, the Fisher information matrix factors
  on the left of each term cancel, yielding the result in the proposition.
\end{proof}

The proof of Proposition~\ref{prop:appendix:svae_natural_gradient} uses the
necessary condition for unconstrained local optimality to simplify
the expression in Eq.~\eqref{eq:svae_natural_gradient_used_stationarity}.
This simplification does not necessarily hold if $\vareta_x$ is
constrained; for example, if the factor $q(x)$ has additional factorization
structure, then there are additional (linear) coordinate subspace constraints
on $\vareta_x$.
Note also that when $q(x)$ is a Gaussian family with fixed covariance (that is,
with sufficient statistics $t_x(x) = x$) the same simplification always applies
because third and higher-order cumulants are zero for such families and hence
$\nabla^3 \log Z_x(\vareta_x) = 0$.

More generally, when the local latent variables have additional factorization
structure, as in the Gaussian mixture model (GMM) and switching linear dynamical
system (SLDS) examples, the natural gradient with respect to $\vareta_\theta$ can
be estimated efficiently by writing Eq.~\eqref{eq:svae_grad_1} as
\begin{align}
  \nabla_{\vareta_\theta} \Lsvae
    &=
    \left( \nabla^2 \log Z_\theta(\vareta_\theta) \right) \left( \prioreta_\theta + \E_{\opt q(x)} \left[ \, (t_{xy}(x, y), 1) \, \right] - \vareta_\theta \right)
    \notag
    \\
    &\quad +
    \nabla \left[ \eta_\theta^\prime \mapsto \L(\vareta_\theta, \vareta_\gamma, \vareta_x^*(\eta_\theta^\prime, \phi)) \right],
\end{align}
where we can recover the second term in Eq.~\eqref{eq:svae_grad_1} by using the
chain rule.
We can estimate this second term directly using the reparameterization trick.
Note that to compute the natural gradient estimate in this case, we need to
apply ${(\nabla^2 \log Z_\theta(\eta_\theta))}^{-1}$ to this term because the
convenient cancellation from
Proposition~\ref{prop:appendix:svae_natural_gradient} does not apply.
When $\vareta_\theta$ is of small dimension compared to $\vareta_\gamma$,
$\phi$, and even $\vareta_x$, this additional computational cost is not large.

\subsection{Estimating the gradients \texorpdfstring{$\nabla_{\phi} \Lsvae$ and $\nabla_{\vareta_\gamma} \Lsvae$}{with respect to the other parameters}}
\label{sec:svae_flat_gradients}

To compute an unbiased stochastic estimate of the gradients $\nabla_\phi
\Lsvae(\vareta_\theta, \vareta_\gamma, \phi)$ and $\nabla_{\vareta_\gamma} \Lsvae(\vareta_\theta, \vareta_\gamma, \phi)$ we use the
reparameterization trick \citep{kingma2013autoencoding}, which is simply to
differentiate a stochastic estimate of the objective
$\Lsvae(\vareta_\theta, \vareta_\gamma, \phi)$ as a function
of $\phi$ and $\vareta_\gamma$.
To isolate the terms that require this sample-based approximation from those
that can be computed directly, we rewrite the objective as
\begin{align}
    \Lsvae(\vareta_\theta, \vareta_\gamma, \phi) = \E_{q(\gamma) \opt q(x)} \log p(y \given x, \gamma) - \KL(q(\theta) q(\gamma) \opt q(x) \, \| \, p(\theta, \gamma, x))
    \label{eq:appendix:svae_objective_with_kl}
\end{align}
where, as before,
\begin{align}
    \opt q(x) \triangleq \exp \left\{ \langle \vareta^*_x(\vareta_\theta, \phi), \; t_x(x) \rangle - \log Z_x(\vareta^*_x(\vareta_\theta, \phi)) \right\}
\end{align}
and so the dependence of the expression in
Eq.~\eqref{eq:appendix:svae_objective_with_kl} on $\phi$ is through
$\vareta^*_x(\vareta_\theta, \phi)$.
Only the first term in Eq.~\eqref{eq:appendix:svae_objective_with_kl} needs to be
estimated with the reparameterization trick.

We summarize this procedure in the following proposition.

\begin{proposition}[Estimating $\nabla_\phi \Lsvae$ and $\nabla_{\vareta_\gamma} \Lsvae$]
\label{prop:svae_flat_gradients}
    Let $\hat\gamma(\vareta_\gamma) \sim q(\gamma)$ and $\hat
    x(\phi) \sim \opt q(x)$ be samples of $q(\gamma)$ and $\opt q(x)$,
    respectively.
    Unbiased estimates of the gradients $\nabla_\phi \Lsvae(\vareta_\theta,
    \vareta_\gamma, \phi)$ and $\nabla_{\vareta_\gamma}
    \Lsvae(\vareta_\theta, \vareta_\gamma, \phi)$ are given by
    \begin{align}
        \nabla_\phi \Lsvae(\vareta_\theta, \vareta_\gamma, \phi)
        &\approx
        \nabla_\phi \log p(y \given \hat x(\phi), \hat \gamma(\vareta_\gamma))
        - \nabla_\phi \KL(q(\theta) \opt q(x) \, \| \, p(\theta, x)),
        \notag
        \\
        \nabla_{\vareta_\gamma} \Lsvae(\vareta_\theta, \vareta_\gamma, \phi)
        &\approx
        \nabla_{\vareta_\gamma} \log p(y \given \hat x(\phi), \hat \gamma(\vareta_\gamma))
        - \nabla_{\vareta_\gamma} \KL(q(\gamma)\, \| \, p(\gamma)).
    \end{align}
    Both of these gradients can be computed by automatically differentiating
    the Monte Carlo estimate of $\Lsvae$ given by
    \begin{align}
        \Lsvae(\vareta_\theta, \vareta_\gamma, \phi)
        \approx
        \log p(y \given \hat x(\phi), \hat\gamma(\vareta_\gamma)) - \KL(q(\theta) q(\gamma) \opt q(x) \, \| \, p(\theta, \gamma, x))
    \end{align}
    with respect to $\vareta_\gamma$ and $\phi$, respectively.
\end{proposition}

\subsection{Partially optimizing \texorpdfstring{$\widehat \L$}{the surrogate objective}
using conjugacy structure}
\label{sec:svae_local_factor_optimization}

In Section~\ref{appendix:svae_objective} we defined the SVAE objective in terms of a
function $\vareta_x^*(\vareta_\theta, \phi)$, which was itself
implicitly defined in terms of first-order stationary conditions for an
auxiliary objective $\widehat\L(\vareta_\theta, \vareta_x, \phi)$.
Here we show how $\widehat\L$ admits efficient local partial optimization in
the same way as the conditionally conjugate model of Section~\ref{sec:conditionally_conjugate_local_optimization}.

In this section we consider additional structure in the local
latent variables.
Specifically, as in
Section~\ref{sec:conditionally_conjugate_local_optimization}, we introduce
to the notation another set of local latent variables $z$ in addition to the
local latent variables $x$.
However, unlike Section~\ref{sec:conditionally_conjugate_local_optimization}, we
still consider general likelihood families $p(y \given x, \gamma)$.

Let $p(z, x \given \theta)$ be an exponential family and $p(\theta)$ be its
corresponding natural exponential family conjugate prior, writing
\begin{align}
    p(\theta) &= \exp \left\{ \langle \prioreta_\theta, \, t_\theta(\theta) \rangle - \log Z_\theta(\prioreta_\theta) \right\},
  \label{eq:appendix:svae_densities_start_2}
    \\
    p(z, x \given \theta) &= \exp \left\{ \langle \prioreta_{zx}(\theta), \, t_{zx}(z, x) \rangle - \log Z_{zx}(\prioreta_{zx}(\theta)) \right\}
    \\
    &= \exp \left\{\langle t_\theta(\theta), \, (t_{zx}(z, x), 1) \rangle \right\}
\end{align}
where we have used $t_\theta(\theta) = \left( \prioreta_{zx}(\theta), -\log Z_{zx}(\prioreta_{zx}(\theta)) \right)$ in Eq.~\eqref{eq:used_conj2}.
Additionally,
let $t_{zx}(z,x)$ be a multilinear polynomial in the statistics $t_z(z)$ and
$t_x(x)$, and
let $p(z \given \theta)$ and $p(x \given z, \theta)$ be
a conjugate pair of exponential families, writing
\begin{align}
    p(z \given \theta) &= \exp \left\{ \langle \prioreta_z(\theta), \, t_z(z) \rangle - \log Z_z(\prioreta_z(\theta)) \right\},
    \\
    p(x \given z, \theta) &= \exp \left\{ \langle \prioreta_x(z, \theta), \, t_x(x) \rangle - \log Z_x(\prioreta_x(z, \theta)) \right\}
    \\
    &= \exp \left\{ \langle t_z(z), \prioreta_x(\theta)^\T (t_x(x), 1) \rangle \right\}.
\end{align}
Let $p(y \given x, \gamma)$ be a general family of densities (not necessarily an
exponential family) and let $p(\gamma)$ be an exponential family prior on
its parameters of the form
\begin{align}
    p(\gamma) = \exp \left\{ \langle \prioreta_\gamma, \,
    t_\gamma(\gamma) \rangle - \log Z_\gamma(\prioreta_\gamma) \right\}.
\end{align}
The corresponding variational factors are
\begin{align}
    q(\theta)
    &=
    \exp \left\{ \langle \vareta_\theta, \, t_\theta(\theta) \rangle - \log Z_\theta(\vareta_\theta) \right\},
    &
    q(\gamma)
    &=
    \exp \left\{ \langle \vareta_\gamma, \, t_\gamma(\gamma) \rangle - \log Z_\gamma(\vareta_\gamma) \right\},
    \notag
    \\
    q(z)
    &=
    \exp \left\{ \langle \vareta_z, \, t_z(z) \rangle - \log Z_z(\vareta_z) \right\},
    &
    q(x)
    &=
    \exp \left\{ \langle \vareta_x, \, t_x(x) \rangle - \log Z_x(\vareta_x) \right\}.
    \notag
\end{align}

As in Section~\ref{appendix:svae_objective}, we construct the surrogate objective
$\widehat\L$ to allow us to exploit exponential family and conjugacy structure.
In particular, we construct $\widehat\L$ to resemble the mean field objective, namely
\begin{align}
    \L(\vareta_\theta, \vareta_\gamma, \vareta_z, \vareta_x)
    \triangleq \E_{q(\theta)q(\gamma)q(z)q(x)} \! \left[ \log \frac{p(\theta)p(\gamma)p(z \given \theta)p(x \given z, \theta) p(y \given x, \gamma)}{q(\theta)q(\gamma)q(z)q(x)} \right],
\end{align}
but in $\widehat\L$ we replace the $\log p(y \given x, \gamma)$ likelihood
term, which may be a general family of densities without much structure,
with a more tractable approximation,
\begin{align}
    \widehat \L(\vareta_\theta, \vareta_z, \vareta_x, \phi)
    &\triangleq \E_{q(\theta)q(z)q(x)} \! \left[ \log \frac{p(\theta)p(z \given \theta)p(x \given z, \theta) \exp \{ \psi(x ; y, \phi) \} }{q(\theta) q(z) q(x)} \right],
\end{align}
where $\psi(x ; y, \phi)$ is a function on $x$ that resembles a conjugate
likelihood for $p(x \given z, \theta)$,
\begin{align}
    \psi(x ; y, \phi) \triangleq \langle r(y ; \phi), \; t_x(x) \rangle,
    \qquad \phi \in \R^m.
\end{align}
We then define $\vareta^*_z(\vareta_\theta, \phi)$ and
$\vareta^*_x(\vareta_\theta, \phi)$ to be local partial
optimizers of $\widehat\L$ given fixed values of the other parameters
$\vareta_\theta$ and $\phi$, and in particular they
satisfy the first-order necessary optimality conditions
\begin{align}
  \nabla_{\vareta_z} \widehat\L(\vareta_\theta,
  \vareta_z^*(\vareta_\theta, \phi),
  \vareta_x^*(\vareta_\theta, \phi), \phi) &= 0,
  &
  \nabla_{\vareta_x} \widehat\L(\vareta_\theta,
  \vareta_z^*(\vareta_\theta, \phi),
  \vareta_x^*(\vareta_\theta, \phi), \phi) &= 0.
\end{align}
The SVAE objective is then
\begin{align}
  \Lsvae(\vareta_\theta, \vareta_\gamma, \phi)
  &\triangleq
  \L(\vareta_\theta, \vareta_\gamma, \vareta_z^*(\vareta_\theta, \phi), \vareta_x^*(\vareta_\theta, \phi)).
  \label{eq:svae_objective_2}
\end{align}

The structure of the surrogate objective $\widehat\L$ is chosen so that it
resembles the mean field variational inference objective for the conditionally
conjugate model of
Section~\ref{sec:conditionally_conjugate_local_optimization}, and as a result
we can use the same block coordinate ascent algorithm to efficiently find
partial optimzers $\vareta_z^*(\vareta_\theta, \phi)$ and
$\vareta_x^*(\vareta_\theta, \phi)$.

\begin{proposition}[Computing $\vareta_z^*(\vareta_\theta, \phi)$ and $\vareta_x^*(\vareta_\theta, \phi)$]
  Let the densities $p(\theta, \gamma, z, x, y)$ and
  $q(\theta)q(\gamma)q(z)q(x)$ and the objectives $\L$, $\widehat\L$, and $\Lsvae$
  be as in Eqs.~\eqref{eq:appendix:svae_densities_start_2}-\eqref{eq:svae_objective_2}.
  The partial optimizers $\vareta^*_z$ and $\vareta^*_x$, defined
  by
  \begin{align}
      \vareta^*_z
      &\triangleq
      \argmax_{\vareta_z} \widehat\L(\vareta_\theta,
      \vareta_z, \vareta_x, \phi),
      &
      \vareta^*_x
      &\triangleq
      \argmax_{\vareta_x} \widehat\L(\vareta_\theta,
      \vareta_z, \vareta_x, \phi)
  \end{align}
  with the other arguments fixed, are are given by
  \begin{align}
    \vareta^*_z
    &=
    \E_{q(\theta)} \prioreta_z(\theta) + \E_{q(\theta) q(x)} \prioreta_x(\theta)^\T (t_x(x), 1),
    &
    \vareta^*_x
    &=
    \E_{q(\theta)q(z)} \prioreta_x(z, \theta) + r(y ; \phi),
    \label{eq:svae_local_update}
  \end{align}
  and by alternating the expressions in Eq.~\eqref{eq:svae_local_update} as updates we can compute
  $\vareta_z^*(\vareta_\theta, \phi)$ and
  $\vareta_x^*(\vareta_\theta, \phi)$ as local partial optimizers of
  $\widehat\L$.
\end{proposition}

\begin{proof}
  These updates follow immediately from Lemma~\ref{lem:optimizing_conjugate_mean_field_factor}.
  Note in particular that the stationary conditions $\nabla_{\vareta_z}
  \widehat\L = 0$ and $\nabla_{\vareta_x} \widehat\L = 0$ yield
  the each expression in Eq.~\eqref{eq:svae_local_update}, respectively.
\end{proof}

The other properties developed in
Propositions~\ref{prop:appendix:svae_lower_bound},~\ref{prop:appendix:svae_natural_gradient},
and~\ref{prop:svae_flat_gradients} also hold true for this model because it is
a special case in which we have separated out the local variables, denoted $x$
in earlier sections, into two groups, denoted $z$ and $x$ here, to match the
exponential family structure in $p(z \given \theta)$ and $p(x \given z,
\theta)$, and performed unconstrained optimization in each of the variational
parameters.
However, the expression for the natural gradient is slightly
simpler for this model than the corresponding version of
Proposition~\ref{prop:appendix:svae_natural_gradient}.

\section{Experiment details and expanded figures}

For the synthetic 1D dot video data,
we trained an LDS SVAE on 80 random image sequences each of length 50, using
one sequence per update, and show the model's future predictions given a prefix
of a longer sequence.
We used MLP image and recognition models each with one hidden layer of 50 units
and a latent state space of dimension 8.

\begin{figure}[p]
  \centering
  \begin{subfigure}{0.95\textwidth}
    \centering
    \includegraphics[width=\textwidth, clip, trim=0cm 0.85cm 0cm 1cm]{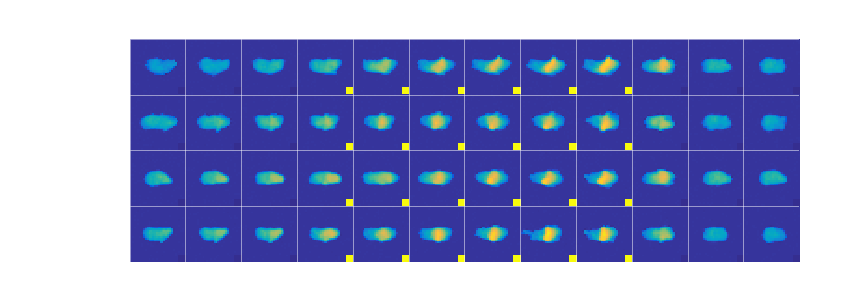}
    \caption{Beginning a rear}

    \vspace{0.3cm}
  \end{subfigure}
  \\
  \begin{subfigure}{0.95\textwidth}
    \centering
    \includegraphics[width=\textwidth]{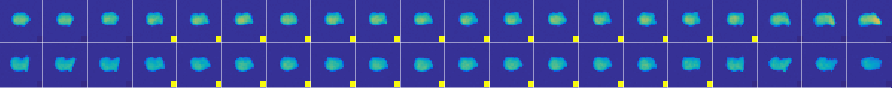}
    \caption{Grooming}

    \vspace{0.3cm}
  \end{subfigure}
  \\
  \begin{subfigure}{0.95\textwidth}
    \centering
    \includegraphics[width=\textwidth, clip, trim=0cm 0.85cm 0cm 1.5cm]{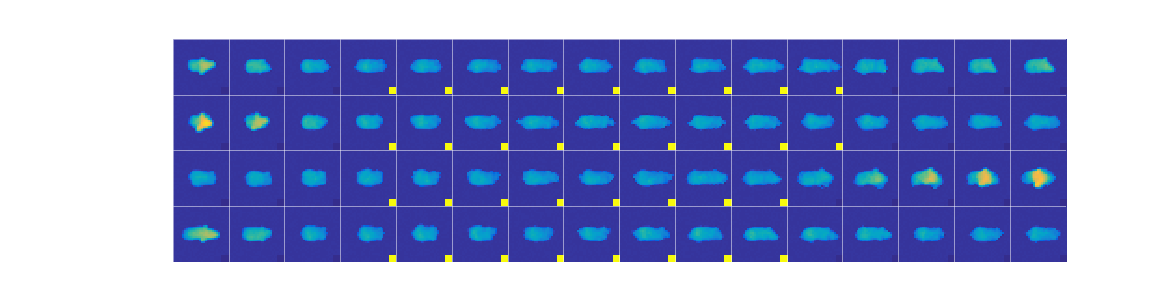}
    \caption{Extension into running}

    \vspace{0.3cm}
  \end{subfigure}
  \\
  \begin{subfigure}{0.95\textwidth}
    \centering
    \includegraphics[width=\textwidth]{figures/26_fall_rear.png}
    \caption{Fall from rear}
  \end{subfigure}
  \caption{Examples of behavior states inferred from depth video. For each
    state, four example frame sequences are shown, including frames during
    which the given state was most probable according to the variational
    distribution on the hidden state sequence. Each frame sequence is padded on
    both sides, with a square in the lower-right of a frame depicting that the
    state was active in that frame. The frame sequences are temporally subsampled
    to reduce their length, showing one of every four video frames. Examples were
    chosen to have durations close to the median duration for that state.}
  \label{fig:slds_svae_syllables_2}
\end{figure}

\end{document}